%% file: main.tex
\documentclass[runningheads]{llncs}
\usepackage{amsmath}
\usepackage[width=122mm,left=12mm,paperwidth=146mm,height=193mm,top=12mm,paperheight=217mm]{geometry}
\usepackage[tocindentauto]{tocstyle}
\input{style/packages.tex}
\DeclareMathOperator{\pa}{pa}

\addtocontents{toc}{\protect\setcounter{tocdepth}{-1}}
\begin{document}

%%ECCV%%%%%%%%%%%%%%%%
\pagestyle{headings}
\mainmatter
\def\SubNumber{614}  % Insert your submission number here
%%%%%%%%%%%%%%%%%%%%%%

%%%%%%%%% TITLE
\def\mytitle{
Feed-forward Uncertainty Propagation in Belief and Neural Networks
}

\title{\mytitle}

%\titlerunning{ECCV-18 submission ID \SubNumber}
%\authorrunning{ECCV-18 submission ID \SubNumber}
\titlerunning{\,}
\authorrunning{\,}

%\author{Anonymous ECCV submission}
\author{Alexander Shekhovtosv \qquad Boris Flach \qquad Michal Busta}
\institute{Czech Technical University in Prague}
%\author{Alexander Shekhovtsov\textsuperscript{1} \qquad Boris Flach\textsuperscript{1}}
%
\maketitle
\begin{abstract}
%------------V3----------------
We propose a feed-forward inference method applicable to belief and neural networks. In a belief network, the method estimates an approximate factorized posterior of all hidden units given the input. In neural networks the method propagates uncertainty of the input through all the layers. In neural networks with injected noise, the method analytically takes into account uncertainties resulting from this noise. Such feed-forward analytic propagation is differentiable in parameters and can be trained end-to-end. Compared to standard NN, which can be viewed as propagating only the means, we propagate the mean and variance. 
The method can be useful in all scenarios that require knowledge of the neuron statistics, e.g. when dealing with uncertain inputs, considering sigmoid activations as probabilities of Bernoulli units, training the models regularized by injected noise (dropout) or estimating activation statistics over the dataset (as needed for normalization methods).
In the experiments we show the possible utility of the method in all these tasks as well as its current limitations.
%
%
%estimation of normalization statistics, processing uncertain inputs, training of stochastic models. Using it for normalization {\em can} improve the training speed and accounting for uncertainty {\em can} improve stability, as demonstrated.
%
%The new model better accounts for probabilistic outputs in the hidden units, which reduces the chance of making predictions that are both wrong and confident. While standard NNs are easily fooled with imperceptible perturbations of the input, the new model is significantly more robust. 
%Estimating uncertainty in regression problems gives a possibility of further integration of different sources of information. We show that estimated uncertainty is well-correlated with the test error --- the model knows when its prediction is inaccurate. 
%Further advantages include the possibility to implement whitening normalization as in batch normalization but without the need of sample-based statistics.
%\keywords{}
\end{abstract}
%
%% some hacks ? %%%%
\setcounter{tocdepth}{2}
\makeatletter
\renewcommand*\l@author[2]{}
\renewcommand*\l@title[2]{}
\makeatletter
%
%%%%%%%%% BODY TEXT
%
\raggedbottom
%\flushbottom
%
%
%%==========================================================================================
\section{Introduction}
In this work we join ideas from graphical models and mainstream NNs and present a feed-forward propagation that one one hand corresponds to an approximate Bayesian inference and on the other is trainable end-to-end. A today's popular view is that statistical and Bayesian methods are not needed and that discriminative end-to-end training is completely sufficient. Let us therefore give examples of problems where statistical tasks arise in NNs.

One important case is when the input is noisy or has some components missing. For many sensors the noise level is known or confidences per measurement are provided (\eg, LIDAR, computational sensors). Also, not all values may be observable all the time (\eg, depth sensors). In \cref{fig:input-noise} we illustrate the point that the average of the network output under noisy input differs from propagating the clean input. Taking into account the uncertainty of the input, an uncertainty of the output can be estimated and used for further processing. In classification networks, propagating the uncertainty of the input can impact the confidence of the classifier and its robustness~\cite{AstudilloN11}. Ideally, we would like that a classifier is not 99.99\% confident when making errors, however such high confidences of wrong predictions are actually observed in NNs~\cite{Moosavi-Dezfooli_2017_CVPR,Fawzi-16-robustness,Szegedy-14-intriguing,nguyen2015deep,Rodner16_FRN}.

Another example is training with dropout~\cite{srivastava14a}. Randomly deactivating neurons can be viewed as multiplicative Bernoulli noise. At the training time, the noise is sampled, so that the learning objective is the expectation of the loss over the noise and the training dataset. However, at test time, the noisy units are replaced by their expected values. Better approximating the expected value of the output may result in a faster training and better test-time performance \cite{wang2013fast}.

Another example is the popular interpretation of sigmoid activations in NNs as probabilities of part detectors and of the whole NN as a hierarchy of such part detectors. If this interpretation is adopted, each inner layer has to deal with its uncertain input defined by the activation probabilities. Considering all hidden units as Bernoulli random variables, we arrive at a sigmoid belief network~\cite{Neal:1992}. With respect to this model one may ask the question whether a specific value computed by the NN indeed represents some probability and how accurate it is. Without a probabilistic model such questions cannot be posed and the probabilistic interpretation becomes purely speculative.

Yet another example is computing expectations of neuron activations when the inputs range over the whole training dataset. It may turn out that the inputs to some non-linearity are always in its saturating part for the whole dataset due to the accumulated bias, or collapse to a single point due to accumulated scaling. Such statistics are crucial in initializing and normalizing NNs~\cite{IoffeS15}. Analytic estimates in networks with random weights were shown to predict well the training and test performance~\cite{Schoenholz2016DeepIP}.

\begin{figure}[t]
\centering
\includegraphics[width=0.8\linewidth]{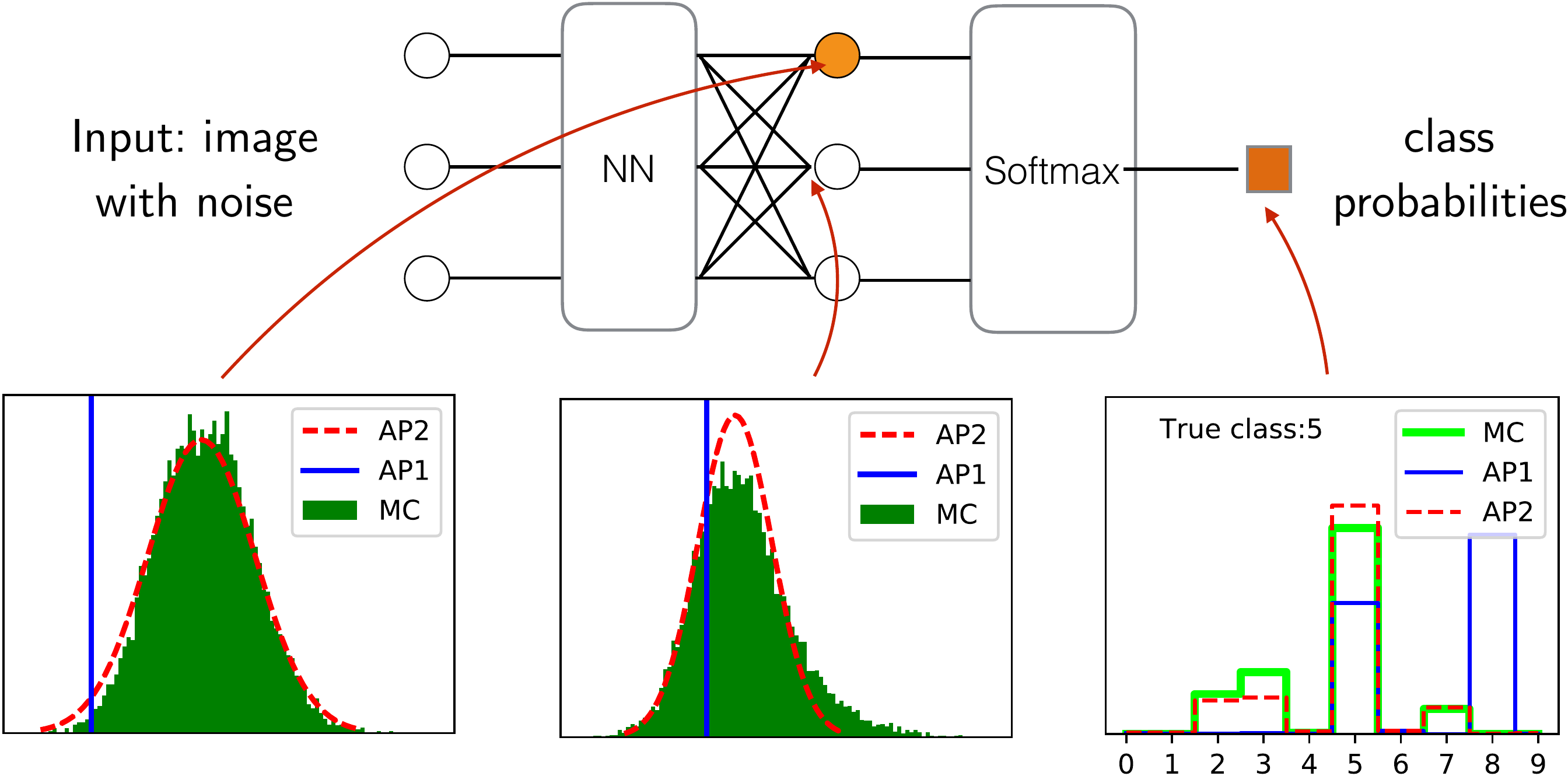}%
\caption{\label{fig:input-noise}
Illustrative example of propagating an input perturbed with Gaussian noise $\N(0,0.1)$ through a fully trained LeNet.
When the same image is perturbed with different samples of noise, we observe on the output empirical distributions shown as MC histograms. Propagating the clean image results in the estimate denoted AP1 which may be away from the MC mean. Propagating mean and variance, as proposed in this paper, results in a posterior Gaussian distribution denoted by AP2 which approximates MC distribution better. For the final class probabilities we approximate the expected value of the softmax. A quantitative evaluation of this case is given in~\cref{sec:experiments}.
%When a network receives a noisy input, propagating the mean values only 
}
\end{figure}

To address the above statistical tasks in NNs we will use Bayesian networks, which are well-established models for reasoning with uncertainty. 
Let us be more precise.
%For clarity, let us recall the definitions of Bayesian and neural networks.
%For clarity, let us be precise about the kinds of networks under consideration.
%
%{\em Bayesian}, or {\em belief}, networks are well-established models for reasoning with uncertainty. 
%{\em Bayesian}, or {\em belief}, networks are well-established models for reasoning with uncertainty. 
{\em Bayesian}, \aka~{\em belief}, networks are composed of random variables $X^k$, each variable being conditionally dependent on its parents through $p(X^k \mid X^{\pa(k)})$. 
%We will restrict the consideration to the case of layered networks with layers $X^0,\dots X^L$ and conditional models $p(X^l \mid X^{l-1}) = \prod_i p(X^l_i \mid {w^l_i}\T X^{l-1})$. 
A special case is a {\em sigmoid belief network}~\cite{Neal:1992} in which random variables are binary-valued and $p(X^k{=1} \mid X^{\pa(k)}) = \S({w^k}\T X^{\pa(k)})$, where $\S$ is the logistic sigmoid function.
A typical Bayesian inference problem consists in determining the posterior distribution of variables of interest given the values of input variables. It implies computing the expectation over all hidden random variables. This marginalization takes into account the uncertainty of all the intermediate hidden variables, but is intractable to compute in general.
{\em Neural networks}, on the other hand, are composite nonlinear mappings of the form $z^k\,{=}\,f^k({w^k}\T\hskip-0.5ex z^{\pa(k)})$, where $z$ are continuous non-random variables and the activation functions $f^k$ are deterministic.
The following connection with sigmoid belief networks exists~\cite{Dayan-95,Flach-17}. 
In the Bayesian inference problem let us approximate the expectation $\E[\S({w^k}\T X^{\pa(k)})]$ with $\S({w^k}\T \E [X^{\pa(k)}])$, \ie, substituting the expectation into the activation function and assuming that the variables $X^{\pa(k)}$ are independent. Associating  the neural network variables $z^k$ with means $\E[X^k]$, we obtain the standard forward propagation in a neural network with sigmoid activation functions.
%
%{\gray
%While neural networks are very successful in solving a vast amount of tasks, the approximate Bayesian inference view reveals some weaknesses. It is often said that hidden units in deep neural networks form a hierarchy of parts, where the output of the activation function represents the posterior probability of a part. Let us for instance consider a unit that has to perform a logical reasoning with uncertain inputs such as "if we see a car wheel with probability 50\% and a car roof with probability 30\%, what is the probability that the two parts are present simultaneously?". In this example (to be detailed~in \cref{sec:sigmoid}), the resulting response may be significantly overconfident when the expectation is approximated as above, \ie, with the standard NN model. We conjecture that the phenomenon of neural networks making grossly overconfident wrong predictions under visually unperceptible perturbations of the input~\cite{Moosavi-Dezfooli_2017_CVPR,Fawzi-16-robustness,Szegedy-14-intriguing,nguyen2015deep,Rodner16_FRN} may be partially attributed to the inaccuracy of this approximation, misrepresenting the uncertainty. Indeed, when substituting the expectation under the activation function, the uncertainty (in particular, the variance) of $X^k$ is completely ignored and only the mean $\E[X^k]$ is used. 
%}
While there exist more elaborate inference methods for belief networks (variational, mean field, Gibbs sampling, \etc), they are computationally demanding and can hardly be applied on the same scale as state-of-the-art NNs.

\paragraph{Contribution}
Our theoretical contribution is the derivation of a feed-forward propagation method for the inference problem in Bayesian networks, connecting the methods developed for graphical models (such as belief propagation) and methods developed for neural networks. Starting from the inference problem formulation, we derive a method that propagates means and variances of all random variables through each consecutive layer. Similar propagation rules were proposed before in a more narrow context~\cite{AstudilloN11,wang2013fast}.

Our technical contribution includes the development of numerically suitable approximations for propagating means and variances for activation functions such as sigmoid, {\tt ReLU}, {\tt max} and, importantly, {\tt softmax} that makes the whole framework practically operational and applicable to a wider class of problems. These technical details are important but postponed to the appendix in the sake of clarity. %given in the appendix due . %.mostly found in the appendix.

The proposed uncertainty propagation enjoys the following useful properties. The model is continuously differentiable even with discontinuous activation functions such as the Heaviside step function, provided that there is some uncertainty. It automatically smoothes out nonlinearities of more uncertain units, allowing better gradient propagation. It takes into account all injected noises, such as dropout, analytically (as opposed to sampling these noises), which is useful both at training and test time. It provides cheaper estimates of statistics over the dataset, can improve training speed, generalization and stability of NNs.

Experimentally, we verify the accuracy of the proposed propagation in approximating the Bayesian posterior and compare it to the standard propagation by NN, which has not been questioned before. This verification shows that the proposed scheme has better accuracy than standard propagation in all tested scenarios. We identify cases where the model scales well and cases in which it has limitations and demonstrate its potential utility in end-to-end learning.
\section{Related Work}\label{sec:related}
{\em Uncertainty propagation} through a multilayer perceptron~\cite{AstudilloN11} has been considered in a limited setting for improving robustness in speech recognition under inputs perturbed with Gaussian noise. %They considered only sigmoid functions in a narrow context with a non-practical approximation that we discuss in~\cref{sec:pae}.
A rather general framework of propagating means and variances was proposed under the name {\em fast dropout training}~\cite{wang2013fast}. A similar feed-forward propagation using Gaussian posterior approximations was proposed as a part of {\em probabilistic backpropagation}~\cite{Hernandez-15-PBP}. %The current work significantly extends these approaches. 
\Wrt these methods ours is a significant extension. We make a connection to Bayesian inference, derive refined approximations, evaluate approximation accuracy and demonstrate a wider range of applications. Another important extension is an analytic propagation through the softmax layer.
%
%Our model and the scope of problems is significantly more general.
%
%We consider a more general model, including deterministic as well as probabilistic dependencies and all common components of convolutional neural networks, we perform end-to-end training and study more phenomena in the context of recent vision models.

Variational inference in {\em nonlinear Gaussian belief networks}~\cite{Frey-99} uses a factorized approximation of the posterior and involves computing means and variances of activation functions such as ReLU. These components are similar to our work but the inference and learning methods are different. %In particular, inference~\cite{Frey-99} requires optimization, while ours is feed-forward. 
The relationship between Bayesian networks and NNs in the context of variational inference has been recently studied in~\cite{Kingma-14-Gradient-based}. These connections are further detailed in~\cref{sec:NLGBN}.
Our method is also related to {\em expectation propagation} in belief networks~\cite{Minka-2001} and {\em loopy belief propagation}~\cite{Pearl-88}. Expectation propagation is derived using forward KL divergence (similar to this work) and was shown equivalent~\cite{Minka-2001} to belief propagation. These methods are iterative and to our knowledge have not been applied to NNs.

{\em Stochastic back-propagation} methods~\cite{Williams1992,Bengio2013EstimatingOP,rezende14} build estimators of the gradient in a network by using samples of noise at the training time and standard propagation at test time.
%is stochastic only at the training time and the estimators use samples of noise. 
We estimate marginal probabilities in the stochastic network both at training and test time analytically. 
%
%{\em Self-normalizing networks}~\cite{Klambauer-SELU} include estimation of mean and variance statistics of activations over the whole dataset. We compute mean and variance of hidden units in order to approximate their posterior for a given input. However, we found it also possible to apply the same model to analytically estimate the statistics over the dataset. While such estimates may be inaccurate, applying them in a normalization scheme the same way as~\cite{IoffeS15} speeds up the learning of our model.
%
Analytic estimates of statistics of hidden units for purposes related to initialization and normalization occurs in~\cite{Klambauer-SELU,ArpitZKG16,Schoenholz2016DeepIP} under the assumption that weights are randomly distributed, which we do not make.
The proposed method may be also relevant in the context of Bayesian model estimation, \ie, inferring the posterior over the parameters given the data~\cite{MacKay-92-Bayesian,Kingma-15-dropout,Hernandez-15-PBP}.

%{\em Batch normalization}~\cite{IoffeS15} and {\em self-normalizing networks}~\cite{Klambauer-SELU} reason about mean and variance statistics of activations over the whole dataset. We compute mean and variance of hidden units in order to approximate their posterior for a given input. However, we found it also possible to apply the same model to analytically estimate the statistics over the dataset. While such estimates may be inaccurate, applying them in a normalization scheme the same way as~\cite{IoffeS15} speeds up the learning of our model.
%The estimates in~\cite{Klambauer-SELU} are also analytic but used for a different purpose: to construct a self-normalizing ELU function. 
%%==========================================================================================
\section{The Method}\label{sec:models}
We consider a Bayesian network organized by layers. There are $l$ layers of hidden random variables $X^k$, $k=1,\dots l$ and $X^0$ is the input layer. Each variable $X^k$ has $n_k$ components (units in layer) denoted $X^k_i$. A conditional {\em Bayesian network} (aka belief network) is defined by the pdf 
\begin{align}\label{DCIM}
p(X^{1,\dots l} \mid X^0) = \textstyle \prod_{k=1}^{l} \prod_{i=1}^{n_k} p(X^k_i \mid X^{k-1}).
\end{align}
The neural network can be seen as a special case of this model by writing mappings such as $Y = f(X)$ as a conditional probability $p(Y \mid X) = \delta(Y - f(X))$, where $\delta$ is the Dirac delta function.
We will further denote values of \rv $X^k$ by $x^k$, so that the event $X^k\,{=}\,x^k$ can be unambiguously denoted just by $x^k$.
%%==========================================================================================
\subsection{Feed-forward Approximate Inference}\label{ff-inference}
In the Bayesian Network~\eqref{DCIM}, the posterior distribution of each layer $k>1$ given the observations $x^0$ recurrently expresses as
\begin{equation}\label{DCIM-forward}
%\begin{align}
p(X^k \mid x^0) = \E_{X^{k-1} |\,x^0\,} \big[p(X^k \mid X^{k-1})\big]
 = \int p(X^k \mid x^{k-1}) p(x^{k-1} \mid x^{0}) \, dx^{k-1}.
%\end{align}
\end{equation}
The posterior distribution of the last layer, $p(X^l \mid x^0)$ is the prediction of the model.

In general, the expectation~\eqref{DCIM-forward} is intractable to compute and the resulting posterior can have a combinatorial number of modes. However, in many cases of interest it might be sufficient to consider a factorized approximation of the posterior $p(X^k \mid x^0) \approx q(X^k) = \prod_i q(X^k_i)$.
We expect that in many recognition problems, given the input image, the hidden states and the final prediction are concentrated around some specific values (unlike in generative problems, where the posterior distributions are typically multi-modal).
The best approximation in terms of forward KL divergence $KL(p(X^k \mid x^0) \| q(X^k))$ is given by the marginals: $q(X^k_i) = p(X^k_i \mid x^0)$, a well-known property explained in~\cref{sec:proofs} for completeness.

The factorized approximation can be computed layer-by-layer, assuming that marginals of the previous layer were already approximated. Plugging the approximation $q(X^{k-1})$ in place of $p(X^{k-1}\mid x^0)$ in~\eqref{DCIM-forward} results in the procedure
\begin{equation}\label{DCIM-Bayes}
 q(X^k_i) = \E_{q(X^{k-1})} \big[ p(X^k_i \mid X^{k-1}) \big]
 = \int p(X^k_i \mid x^{k-1}) \prod_i q(x^{k-1}_i) \, dx^{k-1}.
\end{equation}
This expectation may be still difficult to compute exactly and we will consider suitable approximations later on.

Let us now show how these updates lead to propagation of moments. For binary variables $X^k_i$, occurring in sigmoid belief networks, the distribution $q(X^k_i)$ is fully described by one parameter, \eg, the mean $\mu_i = \E_{q(X^k_i)} [X^k_i] = q(X^k_i{=}1)$. The propagation rule~\eqref{DCIM-Bayes} becomes
%\begin{subequations}\label{moments-binary}
\begin{align}\label{moments-binary}
%& 
\mu_i = \E_{q(X^{k-1})} \big[ p(X^k_i{=}1 \mid X^{k-1}) \big];\ \ \ \ \ \ \
%\label{moments-binary-var}
%& 
\sigma_i^2 = \mu_i(1-\mu_i),
\end{align}
where the variance is dependent but will be needed in propagation through other layers.
%\end{subequations}
For a continuous variable $X^k_i$, we approximate the posterior with a Gaussian distribution $q(X^k_i)$. The closest approximation to the true posterior by a Gaussian distribution $\N(\mu_i, \sigma_i^2)$ w.r.t.~forward KL divergence is given by matching the moments (also detailed in~\cref{sec:proofs}):
\begin{subequations}\label{moments-continuous}
\begin{align}
\mu_i &= \textstyle \int y\,\E_{q(X^{k-1})} [p(X^k_i\=y \mid X^{k-1})] \, dy; \\
\sigma_i^2 &= \textstyle \int y^2\,\E_{q(X^{k-1})} [p(X^k_i\=y \mid X^{k-1})] \, dy - \mu_i^2.
\end{align}
\end{subequations}
When $p(X^k \mid X^{k-1})$ takes the form of a deterministic mapping $X^k = f(X^{k-1})$ as in neural networks, these moments simplify to 
%\begin{subequations}
\begin{align}\label{deterministic-prop}
\mu_i = \E_{q(X^k)} [ f(X^k)]; \ \ \ \ \ \ \
\sigma_i^2 = \E_{q(X^k)} [ f^2(X^k)] - \mu_i^2.
\end{align}
%\end{subequations}
We can therefore represent the approximate inference in networks with binary and continuous variables as a feed-forward moment propagation: given the approximate moments of $X^{k-1} \mid x^0$, the moments of $X^k_i\mid x^0$ are estimated via~\eqref{moments-binary}, \eqref{moments-continuous} ignoring dependencies between $X^{k-1}_j \mid x^0$ on each step (as implied by the factorized approximation).
In the case of categorical variables occurring on the output of classification networks, we consider the full discrete distribution $q(X^k)$. This case cannot be viewed as propagating moments and is treated separately.
%%
%%==========================================================================================
\subsection{Propagation in CNNs}
\paragraph{Linear Layers} The statistics of a linear transform $Y= w\T X$ are given by
\begin{subequations}\label{linear}
\begin{align}
\mu' &= \E[Y] = w\T \E[X] = w\T \mu;\\ %\ \ \ \ \ \ 
\sigma'^2 &= \textstyle \sum_{ij} w_i w_j \Cov[X] \approx \sum_{i} w_i^2 \sigma_i^2,
\end{align}
\end{subequations}
where $\Cov[X]$ is the covariance matrix of $X$. The approximation of the covariance matrix by its diagonal is exact when $X_i$ are uncorrelated. The correlation of the outputs depends on the current weights $w$ and is zero when the weights of the preceding layer are orthogonal. Note that random vectors are approximately orthogonal and therefore the assumption is plausible at least on initialization.
\paragraph{Coordinate-wise mappings}
Let $X$ be a scalar \rv with statistics $\mu, \sigma^2$ and $Y=f(X)$. Assuming that $X$ is distributed as $\N(\mu, \sigma^2)$, we can approximate the expectations~\eqref{deterministic-prop} by analytic expressions for most of the commonly used non-linearities. The derivation will be given in the following subsections. %Towards this end we derived approximations for Heaviside, Logistic transform, ReLU, leaky ReLU shown in and few other non-linearities,
\cref{fig:Heaviside,fig:LReLU} show the approximations derived for propagation through several standard nonlinearities. It includes the case of Bernoulli-logistic unit (\cref{sec:sigmoid}), which has a larger output variance due to its stochasticity. Note that all expectations under Gaussian distribution, unlike the original functions, result in propagation functions that are always smooth.
\begin{figure}[b]
\centering
\includegraphics[width=0.33\linewidth]{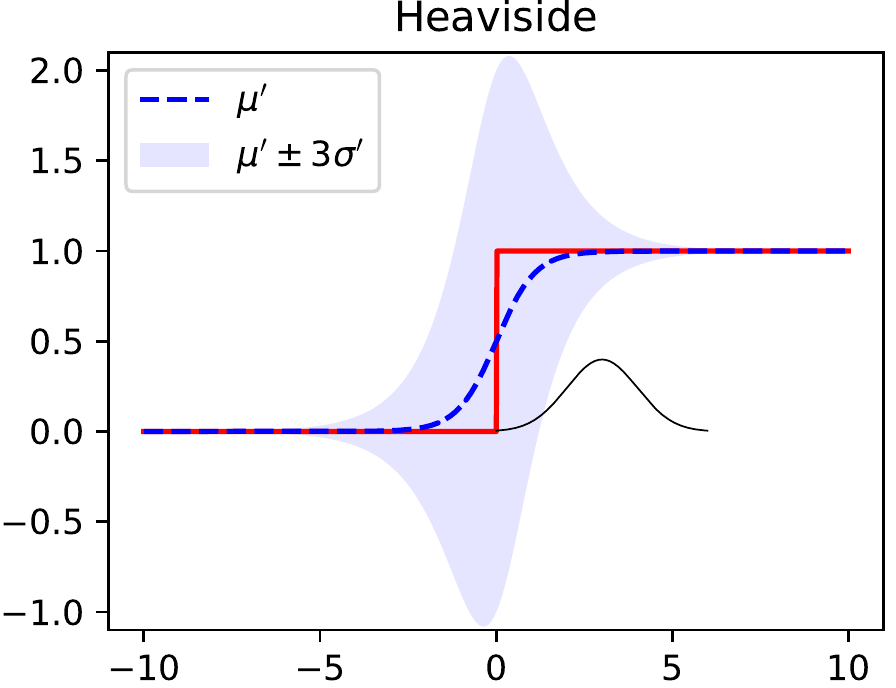}%
\includegraphics[width=0.33\linewidth]{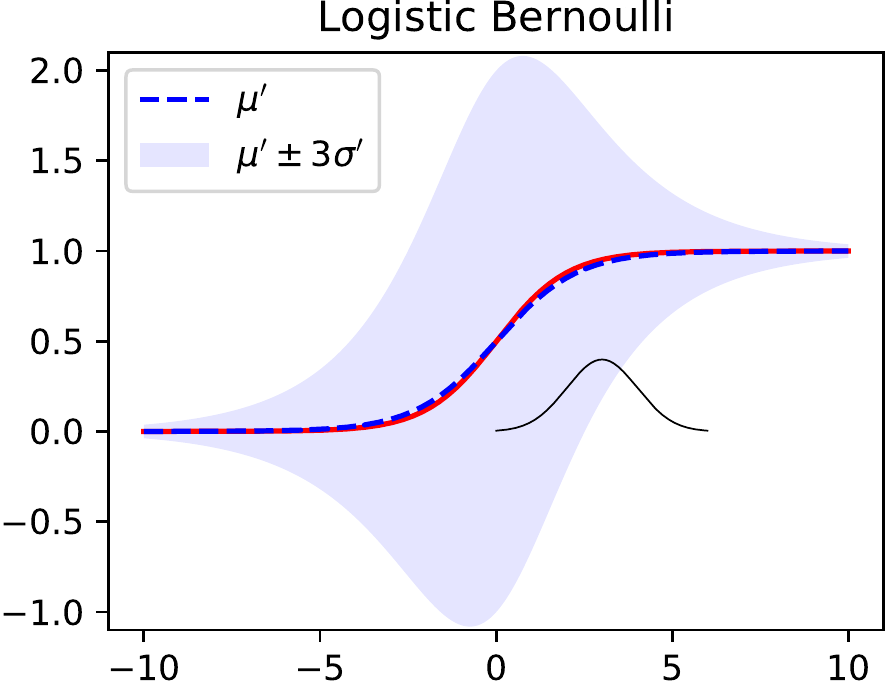}%
\includegraphics[width=0.33\linewidth]{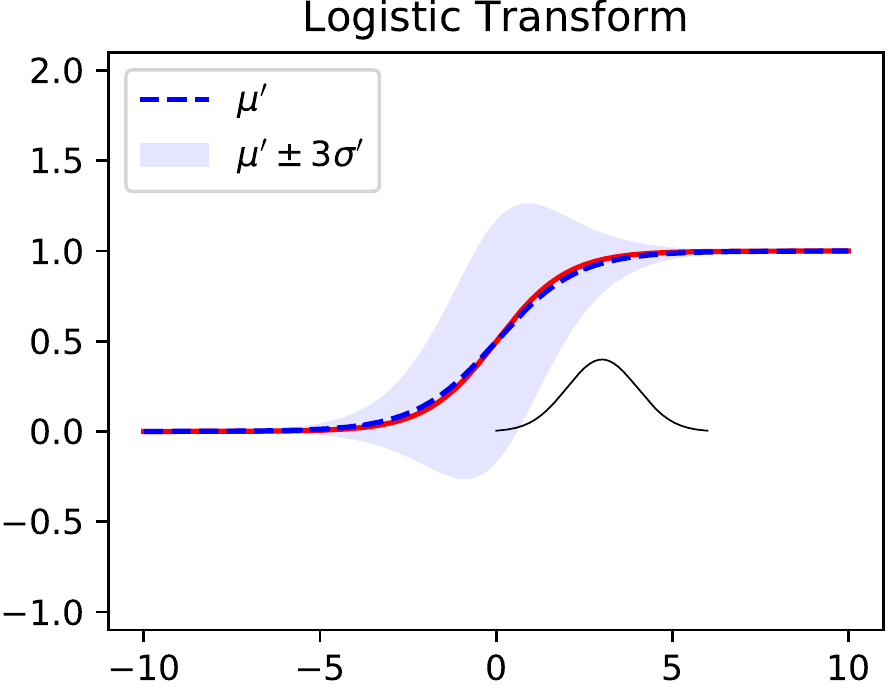}%
\caption{\label{fig:Heaviside}
Propagation for different non-linearities. Heaviside function: $Y =\leftbb X{\geq}0\rightbb$, Bernoulli-logistic: $Y$ is a Bernoulli \rv with $p(Y{=}1 \mid X) = \S(X)$ and logistic-transform: $Y = \S(X)$. Red: activation function. Black: an exemplary input distribution with mean $\mu=3$, variance $\sigma^2 = 1$ shown with support $\mu \pm 3\sigma$. Dashed blue: the approximate mean $\mu'$ of the output versus the input mean $\mu$. The variance of the output is shown as blue shaded area $\mu' \pm 3\sigma'$.
%Mean and variance of the Heaviside step function: $Y = \leftbb X{\geq}0 \rightbb$ using the logistic approximation.
}
\end{figure}
\begin{figure}[b]
\centering
\includegraphics[width=0.33\linewidth]{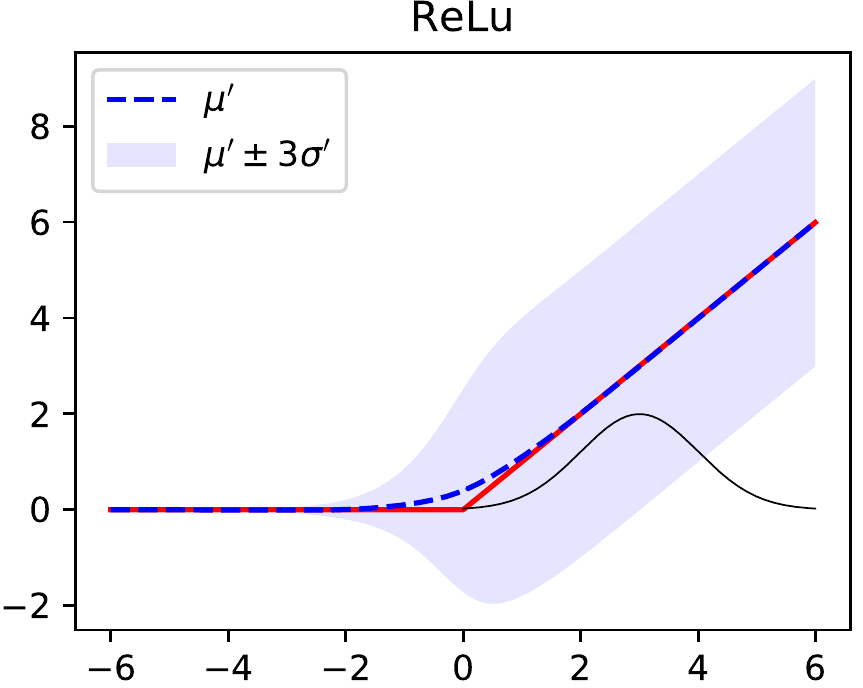}%
\includegraphics[width=0.33\linewidth]{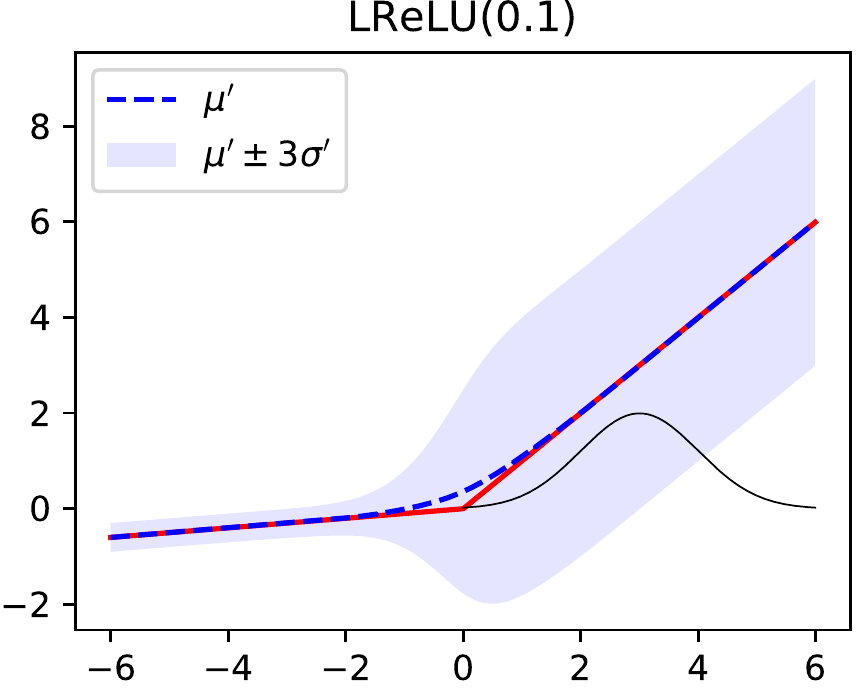}%
\caption{\label{fig:LReLU}
Propagation in ReLU: $\max(0, X)$ and leaky ReLU: $\max(X, \alpha X)$ with $\alpha=0.1$.}
\end{figure}
%
%
%\begin{figure}[t]
%{\centering
%\includegraphics[width=0.9\linewidth]{fig/relu-mean-approx-crop}
%}
%\caption{The mean of $\max(0,X)$, under the approximation that $X$ is distributed Normally or logistically, in both cases with variance 1. The two approximations are very similar. For comparison also the ELU activation~\cite{ClevertUH15} is shown, with an input offset by 1.
%\label{fig:relu}}
%\end{figure}
%
%
\paragraph{Other Layers}
The same principle can be applied to propagation in other layers. %There are known expressions for the product or maximum of 2 random variables.
For example, statistics of the product $X_1 X_2$ of independent \rv's express as
%\begin{subequations}
\begin{align}
\mu' = \mu_1 \mu_2;\ \ \ \ \ \ \
\sigma'^2 = \sigma_1^2 \sigma_2^2 + \sigma_1^2 \mu_2^2 + \mu_1^2 \sigma_2^2.
\end{align}
%\end{subequations}
This is used, \eg, in the propagation through dropout, which can be viewed as multiplicative Bernoulli noise.
We derive an approximation for the maximum of two independent random variables $\max(X_1, X_2)$, which allows to model maxOut and max pooling (in the latter we need to compose the maximum of many variables hierarchically, assuming conditional independence). Closely related to the maximum and crucial in classification problems, is the softmax layer. We derive several approximations to it, of varying complexity and accuracy.
%
%%==========================================================================================
\subsection{Sigmoid Belief Networks}\label{sec:sigmoid}
\paragraph{A Simple Approximation}
Sigmoid belief networks~\cite{Neal:1992} are an important special case in which standard neural networks can be viewed as an approximation to the inference. %We give an instructive example where it leads to wrong estimates and derive the variance propagation.
\par
Let $X$ denote a layer of units and $Y$ denote a single hidden unit with $p(Y{=}1\mid X) = \S(w\T X)$ \footnote{The bias term may be included by assuming, \eg, $X_0=1$.}.
The required propagation~\eqref{moments-binary} expresses as $q(Y{=}1) = \E_{q(X)} \S(w \T X)$.
The first order Taylor series expansion for the mean results in the approximation~\cite{Dayan-95,Flach-17}: 
\begin{align}\label{AP1}
\tag{AP1}
\E_{q(X)} \S(w \T X) \approx \S(w\T \E_{q(X)} X),
\end{align}
\ie, the expectation is substituted under the function.
Using the mean parameters $\mu^X_i = q(X_i{=}1)$, the approximation can be written as $\mu^Y = \S(w\T \mu^X)$ and we recover the standard forward propagation rule in a sigmoid NN.
The following example shows that this approximation may be significantly under- or overestimating.
\begin{example}[Logical AND]\label{example:and}
Consider a unit that has to perform logical reasoning with uncertain inputs such as: if we see a car wheel with probability 50\% and a car roof with probability 30\%, what is the probability that the two parts are present simultaneously?
Let the two binary inputs $X_1, X_2$ have probabilities $p(X_1{=}1)$ and $p(X_2{=}1)$, respectively. We want to approximate the true expectation of $X_1 \wedge X_2$ that both inputs are active. We evaluate the logistic model $p(Y{=}1\mid X) = \S (a(X_1+X_2) + b)$.
The parameters $a,b$ are chosen as follows: for the input $(1,1)$ we require probability $P(Y{=}1\mid X)$ to be at least $1-\varepsilon$ and for inputs $(1,0)$ or $(0,1)$ it must be at most $\varepsilon$. For $\varepsilon=0.05$ this gives $a = 5.89$ and $b = -1.5$, see~\cref{AND-params}. We then compare the exact expectation of the logistic model $\E [Y] = \E_X [p(Y{=}1 \mid X)]$ and the approximation by~\cref{AP1}.
Results in~\cref{tab:AND} show that the expectation of the logistic model fits $\E [X_1 \wedge X_2]$ accurately, while the approximation AP1 is severely underestimating and does not distinguish the cases $(0,1)$ and $(0.5, 0.5)$.
\end{example}
\begin{table}[t]
\centering
\small
\setlength{\tabcolsep}{3pt}
\begin{tabular}{ll|l|l|l}
$p(X_1{=}1)$ & $p(X_2{=}1)$ & $\E [X_1{\wedge}X_2]$ & $\E [Y]$ & AP1\\
\hline
0 & 0 & 0 & 0.00015 & 0.00015 \\
0 & 1 & 0 & 0.05 & 0.05 \\
1 & 1 & 1 & 0.95 & 0.95 \\
0.25 & 0.25 & 0.0625 & 0.077 & 0.0027	 \\
0.5 & 0.5 & 0.25 & 0.26 & 0.05 \\
0.75 & 0.75 & 0.56 & 0.55 & 0.5\\
\end{tabular}
\caption{Logic AND gate in the expectation, Logistic model and standard NN. %, see~\cref{example:and}.
}\label{tab:AND}
\end{table}
This example shows a conceptual flaw in viewing the sigmoid NN as a hierarchy of part detectors since the probabilities estimated by the units in a hidden layer are not taken into account correctly in the next layer.
Approximating the OR gate with a standard NN runs into a similar problem of overestimating the probability. A more accurate approximation can be extremely useful in modeling logic gates and robust statistics of uncertain inputs. %The performance of~\eqref{AP1} cannot be improved by choosing different parameters. It may however be improved by using additional hidden layers fitted to reproduce the required expectation, a significantly more complex model.%
%
%%==========================================================================================
\paragraph{Improved Approximation: Latent Variable Model}\label{sec:new-approx}
We now derive a better approximation for the expectation of the logistic Bernoulli unit making some assumptions about the input. The approximation is essentially similar to~\cite[eq. 8]{MacKay-92-classification}, but we explain the latent variable view, which leads to somewhat different constants an will be important for treating the softmax.

The difficulty in computing $\E_{q(X)} [ \S(w\T X) ]$ stems from the fact that $X_i$ are binary and therefore the activation $U = w\T X$ has a discrete distribution with mass in combinatorially many points. %On the other hand, $\S$ is smooth and is itself a cdf of Logistic distribution. This leads us to the 
Noting that $\S$ is itself the cdf of a logistic distribution, we make the following observation (known as ``latent variable model'' in logistic regression, see also Bernoulli-logistic unit~\protect{\cite[p.4]{Williams1992}})
\begin{observation}\label{O:logistic-threshold}
Let $Z$ be a \rv with cdf $\S$, independent from $X$. 
Then $\E [\S(w\T X)] = \Pr\{ w\T X - Z \geq 0 \}$.
\end{observation}
\begin{proof}
Since $\S(z) = \Pr\{ Z\,{\leq}\,z \}$ we have $\E S(w\T X) = \E_X \Pr\{Z \leq w\T X \mid X \} = \Pr\{ w\T X - Z \geq 0 \}$.
\end{proof}
\par
The density of the difference $w\T X - Z$ is given by the convolution of the discrete density of $w\T X$ with a smooth bell-shaped logistic density of $Z$. This efficiently smooths the discrete density of $U$ and the distribution of $w\T X - Z$ tends to normality when weights $w$ are random or bounded and the dimension increases.
%This has a smoothing effect illustrated in~\cref{fig:density-example} (top)\footnote{Since Logistic density is symmetric about zero, we consider $U+Z$ instead of $U-Z$ in the figure.}, which suggests that the density of $w\T X - Z$ tends to normal with more input variables. 
It allows to use the following approximation:
\begin{proposition}\label{prop:gaussian-input}
Assuming that $w\T X - Z$ has normal distribution, we can approximate
\begin{align}\tag{AP2a}\label{AP2a}
\E_X \S(w\T X) \approx \Phi(\mu/(\sigma^2 + \sigma^2_\S)^{\frac{1}{2}}),
\end{align}
where $\Phi$ is the cdf of standard normal distribution, $\mu = \E (w\T X)$, $\sigma^2 = \Var(w\T X)$ and $\sigma^2_\S = \pi^2/3$ is the variance of the standard logistic distribution.
\end{proposition}
It is obtained as follows. We compute the mean and variance of $V = w\T X - Z$ as $\E [V] = \E [w\T X] - \E[Z]$ and $\Var [V] = \Var[w\T X] + \Var[Z]$. The probability $\Pr\{V > 0\}$ is given by $1-F_{V}(0)$, where $F_{V}$ is the cdf of $V$, which is assumed normal. 
Expressing $1-F_{V}(0)$ through the cdf of the standard normal distribution $\Phi$ we obtain~\cref{AP2a}.
%Further details on this approximation discussed in~\cref{sec:pae}.  
%approximation is illustrated in~\cref{fig:density-example} bottom.
%Further details are discussed in~\cref{sec:pae}.
The following variant is even simpler to compute.
\begin{proposition}\label{P:AP2b}
Assuming that $w\T X - Z$ has logistic distribution, we can approximate
\begin{align}\label{AP2b}
\tag{AP2b}
\E_X \S(w\T X) \approx \S(\mu/s),%.\frac{1}{1 + \exp(\mu/s)},
\end{align}
where $s = (\sigma^2 / \sigma^2_\S + 1)^{\frac{1}{2}}$, $\mu = \E[w\T X]$, $\sigma^2 = \Var(w\T X)$.
\end{proposition}
The scale parameter $s$ is chosen so that the variance of the standard logistic distribution $s^2 \sigma^2_\S$ matches that of $w\T X - Z$.
It is remarkable that~\eqref{AP2b} differs from~\eqref{AP1} only by the scale of the activation. However this scaling is dynamic: it depends on the network parameters and the input.
\par
Illustrations and comparison of accuracy of these approximations are given in~\cref{sec:pae}.
The variance of the logistic Bernoulli unit is defined by~\eqref{moments-binary}.
The logistic transform $Y=\S(X)$ happens to have exactly the same mean, because the mean of the Bernoulli distribution is its probability of drawing $1$. However, the variance of $Y=\S(X)$, illustrated in~\cref{fig:Heaviside}, is different and poses a separate challenge detailed in \cref{detail:logistict-var}.
%This is different from logistic transform. \revisit

%%==========================================================================================
\subsection{General Latent Variable Models}\label{sec:general}
Latent Variable Models, as in~\cref{sec:new-approx}, on one hand help us to compute expectations of some functions and on the other hand form a rich source of stochastic models. This allows to give a universal treatment to sigmoid belief networks, neural networks with uncertain inputs, networks with injected noise and possible combinations thereof.
%
%We now define a more general class of models that can be handled by the approximation technique similar to~\eqref{sec:new-approx}. This class will include sigmoid belief networks, neural networks with uncertain inputs, networks with injected noise and possible combinations thereof.
%
%We elaborate on~\cref{O:logistic-threshold} giving the latent variable view of the conditional probability. 

%As was shown, sigmoid belief network with conditional probabilities $p(Y_j{=}1\mid X) = \S(w\T X)$ can be equivalently represented using deterministic hard threshold activation functions and an additional logistic latent variable (injected noise): $Y_j = \leftbb w\T X + Z \geq 0 \rightbb$, $Z\sim \S$.
We start from the latent variable representation of a sigmoid belief network layer $p(Y_j{=}1\mid X) = \S(w\T X)$, which reads as $Y_j = \leftbb w\T X + Z \geq 0 \rightbb$, $Z\sim \S$. It is straightforward to generalize~\cref{O:logistic-threshold} to other activations that have the form of a cdf, by considering the respective noise. More interestingly, different functions may be used in place of the thresholding function. Consider the following latent variable (injected noise) model:
\begin{align}\label{LVM}
X^k = f(W^k X^{k-1} - Z^{k}),
\end{align}
where $f \colon\Real \to \Real$ is applied component-wise and $Z^k_i$ is an independent real-valued \rv with a known distribution (such as the standard normal distribution). 
%%Injected noise models are often considered explicitly, \eg.~\cite{Bengio2013EstimatingOP}, where the noise is sampled (and added to activations) at training time for regularization and better flow of stochastic gradients. With noise absent, we recover the deterministic NN with general non-linearities. 

From representation~\eqref{LVM} we can reconstruct back the conditional cdf of the belief network $F_{X^k\mid X^{k-1}}(u) = \E_{Z^k} [W^k f(X^{k-1} - Z^k) \leq u ]$ and the respective conditional density.
Examples for the stochastic binary neuron $Y = \leftbb w\T X - Z\rightbb$ considered in~\cite{Williams1992} with general noise and the {\em stochastic rectifier} $Y = \max(w\T X - Z, 0)$ considered in~\cite{Bengio2013EstimatingOP} are given in~\cref{sec:cond-densities}. The conditional density may however be complicated and its explicit form is in fact not needed for our approximate inference. 
The moments of $f(W^k X^{k-1} - Z^{k})$ can be computed directly, provided that the components of $W^k X^{k-1} - Z^{k}$ are assumed to have an approximate distribution density $q$ such that the integral $\int_u q(u) f(u) \: du$ can be computed in closed form. For many non-linearities used in NNs, we can do so with either normal or logistic distribution, similarly to approximations~\cref{AP2a},~\cref{AP2b}.

%We compute moments of $X^k| x^0$ in~\eqref{LVM} using a step-by-step propagation: through the linear transform, added independent noise and the component-wise function. Given approximate moments of $X^{k-1}| x^0$ and assuming $X^{k-1}_i | x^0$ and $Z^k_i$ are all independent, the moments of $U = (W^k X^{k-1} - Z^{k})| x^0$ are easily computed (\cref{table:approx}: linear).

%In order to compute moments of $f^k(U_j)$ we must assume $U_j$ to have an approximate distribution $q(U_j)$ such that the integral $\int_u q(u) f^k(u)$ can be computed in closed form. For many non-linearities used in NNs, we can do so with either normal or logistic distribution, similarly to approximations~\cref{AP2a},~\cref{AP2b}. 
%\cref{table:approx} summarizes the moment propagation equations through a set of layers used in common NN models.

%%==========================================================================================
\subsection{Softmax}
Softmax is the multinomial logistic model: $p(Y{=}y\mid X) = \exp(X_y)/ \sum_{i} \exp(X_i)$. 
Assuming there are $n$ classes, the posterior approximate distribution $q(Y)$ is specified by $n$ numbers. 
The expectation over $X$ in this case is more difficult since it is $n$-dimensional. Fortunately, it is expressible in the latent variable model (\cf~\cref{O:logistic-threshold}) using multidimensional noise:
\begin{align}\label{softmax-lvm}
\textstyle \E [p(Y{=}y\mid X)] = \Pr \{ U-Z \geq 0 \},
\end{align}
where $U$ is a \rv in $\Real^{n-1}$ with components $U_{k} = X_y - X_k$ for $k\neq y$ and $Z$ has an $n{-}1$-variate logistic distribution~\cite{Malik-73}. We make a simplification by using an i.i.d\onedot logistic model for $Z_k$ in~\eqref{softmax-lvm}. We can then build approximations by assuming that $U-Z$ is multivariate normal or multivariate logistic and evaluate the respective cdf instead of $\Pr \{ U-Z \geq 0 \}$. The logistic approximation gives the expression
\begin{align}
q(y) \approx \Big(  1+ \sum_{k\neq y}\exp\Big\{ \frac{\mu_k - \mu_y}{ \textstyle\sqrt{(\smash[b]{\sigma^2_k + \sigma^2_y})/\smash[b]{\sigma^2_\S+1}}  } \Big\}  \Big)^{-1}.
\end{align}
%The resulting approximations are shown in \cref{table:approx} \revisit and 
A drastically simplified approximation, which we use in the end-to-end training, reduces to the {\rm softmax} of $\mu_k/\sqrt{\sigma_k^2/\sigma_S^2+1}$.
%${\rm softmax} (\mu_i\sqrt(\sigma_i +1) )$
%
See the derivations in~\cref{detail:softmax}.
When the input variances are zero, both approximations recover the standard softmax function.
%

%%==========================================================================================
\subsection{End-to-end Training}\label{sec:different-layers}
%The inference procedure obtained from the proposed approximation is very similar to standard neural networks, with the difference that it propagates means and variances. The input can be assumed to have some non-zero variance (coming from a noisy sensor or an estimator with variance). Some further increase of variance may be introduced inside the network by stochastic layers. We now discuss properties of this model, denoted as {\rm varNN} that are relevant in the context of end-to-end learning: differentiability, vanishing of gradients and normalization techniques.

Let $q(X^l\mid x^0; \theta)$ denote the output probabilities of our model. We consider standard learning optimization objectives. For classification, we maximize the conditional likelihood of the training data, \ie, we minimize
\begin{align}\label{log-likelihood}
L = \E_{(x^0, x^{l*})\sim{\rm data}} \big[{-}\log q(X^l=x^{l*}\mid x^{0}; \theta)\big] .
\end{align}
%For regression, we consider either the conditional likelihood as above or the empirical risk of the randomized strategy $q(X^l\mid x^{0})$:
%\begin{align}
%R = \E_{(x^0, x^{l*})\sim{\rm data}} \big[ \E_{q(X^l\mid x^{0}; \theta)} [L(X^l, x^{l*})] \big].
%\end{align}
%In this case the loss function $L(x^l, x^{l*})$ scores the predicted output $x^l$ versus the ground truth $x^{l*}$. Since we predict a posterior probability rather than a single point, we take the expectation over the inferred posterior.

\paragraph{Continuous Differentiability}
The derivative w.r.t.~parameters is obtained by back-propagation expanding derivatives in both mean and variance dependencies. We observe that all propagation equations based on computing expectations \wrt Gaussian or logistic distribution are continuously differentiable assuming non-zero variance. In order to allow for learning with hard non-linearities, such as the Heaviside function, it is sufficient to assume that each input instance has some uncertainty %(\eg. Gaussian noise in color images) 
or that the first layer is stochastic. %No further heuristic smoothing is necessary. 
Typically used heuristics such as replacing the step function with identity in the backward pass can be avoided.

%\paragraph{Self-Scaling}
%Note that in all approximations for different non-linear transforms in~\cref{table:approx}, the mean is always divided by the square root of variance whenever it occurs in saturating functions such as $\Phi,\phi,\S$. Let us consider a linear transform layer $w\T X+b$ with mean $\mu' = w\T \mu + b$ (now with explicit bias),  and variance $\sigma'^2 = \sum_{i}w_i^2 \sigma_i^2$. The ratio $\mu'/\sigma'$ is invariant to scaling $w$ and $b$ by the same constant. It follows that the inputs to saturating functions are invariant to this scaling and 
%It follows that networks with ReLU and LReLU activations are equivariant to this scaling: the outputs of these activations are scaled proportionally. %while networks with Heaviside activations are invariant.
%The non-linearities following a linear layer are then equivariant to this scaling: 
%Since all saturating functions in are based on this ratio, it means that the network is to some degree self-normalizing. 
%However, since stochastic layers are adding noise with constant variance, the scale is still important.

\paragraph{Gradients}
In the approximations for different non-linear transforms such as~\cref{AP2b} (full list in \cref{sec:approximations}), the mean is always divided by the standard deviation whenever it occurs in saturating functions such as $\Phi,\phi,\S$.
When hidden units are uncertain, the variance is larger and all non-linearities automatically become smoother. The automatic scaling by uncertainty allows the gradient to be propagated deeper through the network. The space of parameters where $\mu/\sigma$ is small, as opposed to where $\mu$ is small in standard networks, and the gradients do not vanish is different and may be larger. %For training, the network should be initialized to have sufficiently large uncertainties in all units. A random initialization with an activation whitening achieves this condition. During the training with the maximum likelihood criterion, there is a natural high penalty for making confident but wrong predictions. We expect that uncertainties in the network will stay high until the training data can be classified with high accuracy, with ultimately only wrongly classified instances having larger uncertainties.
%
%%==========================================================================================
%\input{tex/experiments.tex}
\input{tex/exp1.tex}
%%==========================================================================================
\section{Conclusion}
We have described an inference method which lies between feed-forward neural networks and iterative inference methods such as variational methods or belief propagation in Bayesian networks. We have build a framework of variance-propagating layers, extending constructive elements of standard NNs, in which a range of models can be considered with deterministic and stochastic units and used in end-to-end learning.
The feed-forward structure is one one hand restrictive, because we can only perform inference in one direction. On the other hand, it allows dealing with uncertainties in NNs and opens a number of possibilities with practical benefits. The quality of the approximation of posterior probabilities can be measured. The accuracy is sufficient for several use cases such as sigmoid belief nets, dropout training and normalization techniques. It may be insufficient for propagating input uncertainties through a deep network, but we believe that a calibration will be possible. Further applications may include generative and semi-supervised learning (VAE) and Bayesian model estimation.
%
%
%
%In the analysis of simple cases such as logic gates and a single hidden layer, we showed that taking uncertainty into account significantly improves model reliability while not compromising accuracy in the ideal case. These properties carry over to deeper models, but only up to a point, suffering from overfitting. In contrast, standard model can be very sensitive to perturbations even when overfitting does not occur.
%We also showed that our framework can compute statistics needed for normalization or initialization, and obtained improved performance over BN. 
%
%There are indeed many more properties that can be further explored. For example estimating uncertainty in regression problems. %, initializing training of standard NN, applications in generative models.
%
%
%We explored two important properties - computing normalization / parameter initialization and smoothly incorporating stochastic units without need of sampling. There are much more to be explored. 
%The are indeed much more possibilities that this work opens: 
%Estimating uncertainty in regression problems gives a possibility of further integration of different sources of information. We show that estimated uncertainty is well-correlated with the test error --- the model knows when its prediction is inaccurate. 
%It only performs inference in one direction, 
%
%
\section*{Acknowledgment}
A.~Shekhovtsov was supported by Toyota Motor Europe HS and Czech Science Foundation grant 18-25383S. B.~Flach was supported by Czech Science Foundation grant 16-05872S.
{\small
\renewcommand{\bibname}{\protect\leftline{References\vspace{-2ex}}}
\bibliographystyle{splncsnat}
\bibliography{../bib/strings,../bib/neuro-generative,../bib/our}
}

%\newpage
\input{tex/appendix.tex}
\end{document}

%% file: style/packages.tex
\usepackage{etex}
\usepackage{url}
\usepackage[numbers]{natbib}
\renewcommand\bibname{References}

\usepackage{cvpr-abbriv}
\usepackage{epsfig}
\usepackage{graphicx}
\usepackage{amsmath}
\usepackage{amssymb}
\usepackage{array}
\usepackage{tabularx}
\usepackage{url}
\usepackage{graphicx}
\usepackage{amsmath,amssymb} % define this before the line numbering.
\makeatletter
\let\proof\@undefined                        % undefine \proof
\let\endproof\@undefined                  % undefine \endproof
\makeatother

\usepackage{amsthm}
\usepackage{amsthm}
\usepackage{amsxtra}
\usepackage{stmaryrd}
\usepackage{comment}
\usepackage[ruled,vlined,linesnumbered,procnumbered,longend]{algorithm2e}

\SetCommentSty{mycommfont}
\usepackage{multirow,rotating}
\usepackage{array}
\usepackage{afterpage}
\usepackage{dblfloatfix}
\usepackage{setspace}
\usepackage{floatrow}
\usepackage{booktabs}
\usepackage{color, colortbl}
\usepackage{tabularx,calc}
\usepackage{tikz}                 % create images using pgf/tikz
\usetikzlibrary{patterns,fit,external}
\usepackage{pgfplots}             % draw plots in tikz style
\usepackage{overpic}
\usepackage{bbm}
\usepackage[absolute]{textpos}
\usepackage{atbeginend}
\usepackage{units}
\usepackage{xifthen}
\usepackage{chngcntr}

\usepackage{times}
\usepackage{tikzscale}
\usepackage{marginnote}
\usepackage[textwidth=2cm,figwidth=\columnwidth,colorinlistoftodos]{todonotes}
\makeatletter
\renewcommand{\todo}[2][]{\tikzexternaldisable\@todo[#1]{#2}\tikzexternalenable}
\makeatother

\newcounter{mycomment} % Usage:  \mycomment[CR]{Schreib was gscheits}

\usepackage{booktabs}
\usepackage{multirow}
\usepackage{tabularx}
\usepackage{array}
\usepackage{floatrow}
\usepackage[
  position=bottom,
  font=small,
  labelfont={bf},
]{caption,subfig}
\usepackage{longtable}
\usepackage{tabu} %% !!! this is better!
\usepackage[most]{tcolorbox}
\usepackage{pbox}
\newlength{\luw}
\newlength{\luh}

\usepackage{xifthen}

\usepackage[draft=false, pagebackref=false,breaklinks=true,colorlinks,bookmarks=false,pdftex]{hyperref}
\hypersetup{hidelinks,colorlinks=false,linkbordercolor={1 1 1}}
\input{style/tex_defs.tex}
\input{style/myspace.tex}
\usepackage{thmtools, thm-restate}
\usepackage[capitalise,nameinlink]{cleveref}
\crefformat{equation}{(#2#1#3)}
\crefname{section}{\parSym}{\parSym\parSym}
\Crefname{section}{\parSym}{\parSym\parSym}
\crefname{appendix}{\parSym}{\parSym\parSym}
\crefname{observation}{Observation}{Observations}
\usepackage[bottom]{footmisc}

\makeatletter
\renewcommand\tableofcontents{%
    \@starttoc{toc}%
}
\makeatother

%% file: style/tex_defs.tex
\DeclareMathAlphabet{\mathcalligra}{T1}{calligra}{m}{n}
\DeclareMathAlphabet{\mathantt}{OT1}{antt}{li}{it}
\DeclareMathAlphabet{\mathpzc}{OT1}{pzc}{m}{it}

\DeclareMathOperator{\logsumexp}{logsumexp}
\DeclareMathOperator{\Var}{Var}
\DeclareMathOperator{\Cov}{Cov}
\DeclareMathOperator{\Li}{Li}

\DeclareMathOperator{\relu}{ReLU}
\DeclareMathOperator{\Heviside}{Heviside}

\newcommand{\argmax}{\mathop{\rm argmax}}

\renewcommand{\mid}{\,|\,}

\newcommand{\Real}{\mathbb{R}}

\newcommand{\R}{\mathcal{R}}

\newcommand{\E}{\mathbb{E}}

\newcommand{\N}{\mathcal{N}}

\renewcommand{\L}{\mathcal{L}}

\def\T{^{\mathsf T}}

\newcounter{myRomanCounter}

\newcommand{\x}{{x}}

\def\c{\textsc{c}}
\def\d{{\rm d}}

\newcommand{\gray}{\color[rgb]{0.5,0.5,0.5}}
\newcommand{\red}{\color[rgb]{1,0,0}}

\renewcommand*{\paragraph}[1]{\par\noindent{\normalsize\bf #1}\,\xspace}

\def\mathrlap{\mathpalette\mathrlapinternal} 
\def\mathllap{\mathpalette\mathllapinternal}
\def\mathllapinternal#1#2{\llap{$\mathsurround=0pt#1{#2}$}}
\def\mathrlapinternal#1#2{\rlap{$\mathsurround=0pt#1{#2}$}}

\def\ij{i,j}

\def\leftbb{\mathrlap{[}\hskip1.3pt[}
\def\rightbb{]\hskip1.36pt\mathllap{]}}

\def\epsilon{\varepsilon}

\newcommand{\IF}{\mbox{ \rm if }}

\newcommand{\AND}{\mbox{ \rm and }}

\makeatletter
\let\parSym\S
\def\S{\mathcal{S}}
\makeatother



\def\cf{\emph{cf}\onedot} 

\newcommand{\revisit}[1][]{%
\ifthenelse{\equal{#1}{}}{% no argument
\ensuremath{\red \triangle}\xspace}{%
{\ensuremath{\red \rhd}\xspace}%
{\gray #1}%
{\ensuremath{\red \lhd}\xspace}%
}%
}

\def\anchor [#1]#2{%
\phantomsection{}#1\label{#2}%
\def\arga{#2}%
\global\expandafter\def\csname#2\endcsname{%
\hyperref[#2]{#1}\xspace%
}%
}%

\def\codefunction [#1]#2{%
\phantomsection{}\label{#2}{\ttfamily #1\xspace}%
\def\arga{#2}%
\global\expandafter\def\csname#2\endcsname{%
\hyperref[#2]{\ttfamily #1}\xspace%
}%
}

\usepackage{array}
\newcolumntype{L}[1]{>{\raggedright\let\newline\\\arraybackslash\hspace{0pt}}m{#1}}
\newcolumntype{C}[1]{>{\centering\let\newline\\\arraybackslash\hspace{0pt}}m{#1}}
\newcolumntype{R}[1]{>{\raggedleft\let\newline\\\arraybackslash\hspace{0pt}}m{#1}}

\SetKwFor{forever}{while true}{}{end while}

\def\epsilon{\varepsilon}

\def\={\,{=}\,}

%% file: style/myspace.tex
\newlength{\myskip}
\def\skipmyskip{\vspace{0.5\myskip}}

\SetAlgoSkip{skipalgomyskip}
\SetAlgoInsideSkip{}
\makeatletter
\let\corollary\@undefined
\let\c@corollary\@undefined
\let\endcorollary\@undefined
\let\definition\@undefined
\let\c@definition\@undefined
\let\enddefinition\@undefined
\let\proof\@undefined
\let\endproof\@undefined
\let\theorem\@undefined
\let\c@theorem\@undefined
\let\endtheorem\@undefined
\let\lemma\@undefined
\let\c@lemma\@undefined
\let\endlemma\@undefined
\let\example\@undefined
\let\c@example\@undefined
\let\endexample\@undefined
\let\remark\@undefined
\let\c@remark\@undefined
\let\endremark\@undefined
\let\proposition\@undefined
\let\c@proposition\@undefined
\let\endproposition\@undefined
\let\property\@undefined
\let\endproperty\@undefined
\makeatother

\usepackage{thmtools}

\newtheoremstyle{tightItalic}% name
  {0.5\myskip}%      Space above
  {0\myskip}%      Space below
  {}%         Body font
  {}%         Indent amount (empty = no indent, \parindent = para indent)
  {\itshape}% Thm head font
  {.}%        Punctuation after thm head
  { }%     Space after thm head: " " = normal interword space;
  {}%         Thm head spec (can be left empty, meaning `normal')

\newtheoremstyle{tightBf}% name
  {0.5\myskip}%      Space above
  {0.5\myskip}%      Space below, 0 since amsart starts a paragraph
  {}%         Body font
  {}%         Indent amount (empty = no indent, \parindent = para indent)
  {\bf}% Thm head font
  {.}%        Punctuation after thm head
  {.5em}%     Space after thm head: " " = normal interword space;
  {}%         Thm head spec (can be left empty, meaning `normal')

\theoremstyle{definition}
\theoremstyle{tightBf}

\declaretheorem[style=tightBf,name=Theorem]{theorem}
\declaretheorem[style=tightBf,name=Lemma]{lemma}
\declaretheorem[style=tightBf,name=Definition]{definition}
\declaretheorem[style=tightBf,name=Corollary]{corollary}
\declaretheorem[style=tightBf,name=Proposition]{proposition}
\declaretheorem[style=tightBf,name=Observation]{observation}

\declaretheorem[style=tightBf,name=Example]{example}

\declaretheorem[style=tightItalic,name=Remark]{remark}

\theoremstyle{tightItalic}
\newtheorem*{proof}{Proof}
\BeforeEnd{proof}{\qed}

%% file: tex/exp1.tex
\section{Experiments}\label{sec:experiments}
\paragraph{Implementation}
Our implementation (in pytorch) will be made available upon acceptance. Important for our experiments, the implementation is modular: with each of the standard layers we can do 3 kinds of propagation: {\em AP1}: standard propagation in deterministic layers and taking the mean in stochastic layers (\eg, in dropout we need to multiply by the Bernoulli probability), {\em AP2}: proposed propagation rules with variances and {\em sample}: by drawing samples of any encountered stochasticity (such as sampling from Bernoulli distribution in dropout). The last method is also essential for computing Monte Carlo (MC) estimates of the statistics we are trying to approximate. When the training method is sample, the test method is assumed to be AP1, which matches the standard practice of dropout training. Details of the implementation and models are given in~\cref{sec:extraexp}.
%
%%==========================================================================================
%
\begin{table}[b]
\centering
\setlength{\tabcolsep}{3pt}
\resizebox{\linewidth}{!}{
\begin{tabular}{cc}
\begin{tabular}{|c|c|cc|cc|cl|}
\hline
& CA%
& C%
& A%
& C%
& A%
& F%
& Softmax\\%
\hline
\multicolumn{8}{|c|}{LReLU, Noisy input $\N(0,10^{-4})$}\\
\hline
$\epsilon_{\mu_1}$ & 0.02 & 0.04 & 0.04 & 0.06 & 0.05 & 0.08 & KL 1.4e-6\\
\hline
$\epsilon_{\mu_2}$ & 0.02 & 0.02 & 0.02 & 0.03 & 0.02 & 0.03 & KL 4.5e-7 \\
$\epsilon_{\sigma_2}$ & 1.01 & 0.88 & 0.87 & 0.64 & 0.63 & 0.63 & KL' 3.8e-7\\
\hline
\multicolumn{8}{|c|}{LReLU, Noisy input $\N(0,0.01)$}\\
\hline
$\epsilon_{\mu_1}$ & 0.14 & 0.29 & 0.24 & 0.60 & 0.45 & 0.67 & KL 0.03\\
$\epsilon_{\mu_2}$ & 0.02 & 0.03 & 0.03 & 0.05 & 0.05 & 0.08 & KL 0.003 \\
$\epsilon_{\sigma_2}$          & 1.07 & 0.91 & 0.91 & 0.68 & 0.58 & 0.69 & KL' 0.002\\
\hline
\multicolumn{8}{|c|}{LReLU, Noisy input $\N(0,0.1)$}\\
\hline
$\epsilon_{\mu_1}$ & 0.31 & 0.74 & 0.51 & 1.47 & 1.11 & 1.60 & KL 5.09\\
\hline
$\epsilon_{\mu_2}$ & 0.03 & 0.04 & 0.05 & 0.06 & 0.07 & 0.10 & KL 0.07 \\
$\epsilon_{\sigma_2}$          & 1.12 & 0.99 & 1.09 & 0.75 & 0.65 & 0.79 & KL' 0.03\\
\hline
\end{tabular}
\begin{tabular}{|c|c|cc|cc|cl|}
\hline
& CA%
& C%
& A%
& C%
& A%
& F%
& Softmax\\%
\hline
\multicolumn{8}{|c|}{LReLU, Dropout (0.2)}\\	
\hline
$\epsilon_{\mu_1}$   & 0.01 & 0.02 & 0.04 & 0.17 & 0.06 & 0.10 & KL  0.02\\
$\epsilon_{\mu_2}$ & 0.01 & 0.02 & 0.01 & 0.03 & 0.02 & 0.03 & KL  9.5e-3 \\
$\epsilon_{\sigma_2}$          & 1.00 & 1.00 & 1.01 & 0.95 & 0.98 & 0.95 & KL' 2.9e-3 \\
\hline
\multicolumn{8}{|c|}{Logistic Bernoulli}\\
\hline
$\epsilon_{\mu_1}$   & 0.02 & 0.02 & 0.13 & 0.20 & 0.20 & 0.17 & KL 6.0 \\
$\epsilon_{\mu_2}$ & 0.02 & 0.02 & 0.03 & 0.04 & 0.08 & 0.07 & KL  0.06 \\
$\epsilon_{\sigma_2}$          & 1.09 & 1.00 & 1.00 & 0.69 & 0.95 & 0.93 & KL' 1.2e-2 \\
\hline
\multicolumn{8}{|c|}{Logistic Transform, Noisy input $\N(0,0.1)$}\\
\hline
$\epsilon_{\mu_1}$   & 0.06 & 0.17 & 0.17 & 0.40 & 0.38 & 0.53 & KL 0.01 \\
$\epsilon_{\mu_2}$ & 0.03 & 0.05 & 0.05 & 0.12 & 0.12 & 0.17 & KL  2.6e-3 \\
$\epsilon_{\sigma_2}$          & 1.00 & 1.10 & 1.09 & 0.72 & 0.72 & 0.90 & KL' 1.4e-2 \\
\hline
\end{tabular}
\end{tabular}
}
\caption{\label{tab:accuraacy-LeNet}
Accuracy of approximation for LeNet/MNIST Model per layer. Columns: C - conv layer, A - activation, F - fully connected. In the case of dropout, it is included in the activation. Rows: $\epsilon_{\mu_1}$ - error of approximation of the mean by the standard method (AP1), $\epsilon_{\mu_2}, \epsilon_{\sigma_2}$ errors of approximation by the proposed method (AP2). Errors in $\mu$ are relative to average MC $\sigma^*$, errors in $\sigma$ show the multiplicative factor of $\sigma^*$ (see text). In the final layer we show KL divergence of the class posterior from MC class posterior. For AP2 two versions are evaluated: using simplified approximation of softmax and full approximation of softmax (KL').
}
\end{table}
\paragraph{Approximation Accuracy}\label{sec:exp-accuracy}
%To measure how accurately we can approximate the mean and variance of the posterior $p(X^k \mid x^0)$ we compared the model against the MC.
\cref{tab:accuraacy-LeNet,tab:accuracy-CIFAR} report approximation accuracy per layer in LeNet and CIFAR networks for different use cases.
All models with {\tt LReLU} use $\alpha=0.01$. We have computed MC statistics $\mu^*, \sigma^*$ per unit in each layer.
The error measure of the means $\epsilon_\mu$ is the average $|\mu-\mu^*|$ relative to the average $\sigma^*$. The averages are taken over all units in the layer and over input images. The error of the standard deviation $\epsilon_\sigma$ is the geometric mean of $\sigma/\sigma^*$, representing the error as a factor from the true value (\eg, $1.0$ is exact, $0.9$ is under-estimating and $1.1$ is over-estimating). MC estimates are using $10^3$ samples, which was sufficient to compute 2 significant digits as reported.

The following can be observed from the results in~\cref{tab:accuraacy-LeNet,tab:accuracy-CIFAR}.
The propagation of the input uncertainty works reasonably well for LeNet network (4 pairs of linear and activation layers) but degrades with more layers as seen for CIFAR network in~\cref{tab:accuracy-CIFAR}. This is to be expected since the errors of the approximation accumulate. The main contribution to the loss of accuracy comes from the poor approximation of the variance in the convolutional layers, where we  ignored dependencies. This is clearly seen in the case of small input noise where the accuracy in $\sigma$ drops significantly after the convolutional layers. It means that the inputs are positively correlated on average so that ignoring this correlation underestimates the variance.

Differently from propagating input uncertainty through deep networks, in the case of dropout and Bernoulli models the uncertainty created by a layer appears to dominate the uncertainty propagated from the preceding layers and thus the estimation does not degrade.
%
%
%We conclude that the framework is readily applicable to stochastic models such as dropout and Bernoulli and for normalization techniques.
%
%%
%
%
%
%
\begin{table}[t]
\centering
\setlength{\tabcolsep}{3pt}
\resizebox{\linewidth}{!}{
\begin{tabular}{|c|cc|cc|cc|cc|cc|cc|cc|cc|cclc|}
\hline
& C%
& A%
& C%
& A%
& C%
& A%
& C%
& A%
& C%
& A%
& C%
& A%
& C%
& A%
& C%
& A%
& C%
& P%
& \multicolumn{2}{c|}{Softmax}\\%
%\multicolumn{21}{|c|}{LReLU, Noisy input $\N(0,10^{-4})$}\\
%\hline
%$\epsilon_{\mu_1}$ & 0.02 & 0.04 & 0.07 & 0.10 & 0.20 & 0.19 & 0.37 & 0.26 & 0.52 & 0.32 & 0.64 & 0.37 & 0.71 & 0.47 & 0.70 & 0.46 & 0.60 & 0.70 &
%KL & 0.01 \\
%\hline
%$\epsilon_{\mu_2}$ & 0.02 & 0.02 & 0.02 & 0.05 & 0.15 & 0.15 & 0.31 & 0.22 & 0.44 & 0.28 & 0.55 & 0.32 & 0.60 & 0.40 & 0.59 & 0.40 & 0.51 & 0.59 &
%KL & 0.01 \\ 
%$\epsilon_{\sigma_2}$ & 1.00 & 1.02 & 0.54 & 0.45 & 0.34 & 0.27 & 0.24 & 0.19 & 0.15 & 0.12 & 0.09 & 0.07 & 0.06 & 0.04 & 0.04 & 0.03 & 0.03 & 0.04 & 
%KL' & 0.01\\
%\hline
\hline
\multicolumn{21}{|c|}{LReLU, Noisy input $\N(0,0.01)$}\\
\hline
$\epsilon_{\mu_1}$ & 0.02 & 0.24 & 0.39 & 0.37 & 0.76 & 0.63 & 1.13 & 0.85 & 1.59 & 1.10 & 1.93 & 1.32 & 2.65 & 1.86 & 2.76 & 1.93 & 1.99 & 2.59 & 
KL & 4.42 \\
\hline
$\epsilon_{\mu_2}$ & 0.02 & 0.02 & 0.03 & 0.19 & 0.49 & 0.36 & 0.55 & 0.42 & 0.73 & 0.55 & 0.96 & 0.68 & 1.26 & 0.97 & 1.27 & 1.01 & 1.14 & 1.43 & 
KL & 1.71 \\ 
$\epsilon_{\sigma_2}$ & 1.00 & 1.11 & 0.59 & 0.40 & 0.37 & 0.28 & 0.35 & 0.19 & 0.26 & 0.10 & 0.17 & 0.05 & 0.11 & 0.05 & 0.08 & 0.04 & 0.07 & 0.07 & 
KL' & 1.71\\
%\multicolumn{21}{|c|}{LReLU, Noisy input $\N(0,0.1)$}\\
%\hline
%$\epsilon_{\mu_1}$ & 0.02 & 0.40 & 0.61 & 0.43 & 1.32 & 0.65 & 1.35 & 0.92 & 1.73 & 1.26 & 2.28 & 1.55 & 3.22 & 2.32 & 3.03 & 2.32 & 2.75 & 3.67 & 
%KL & 6.09 \\
%\hline
%$\epsilon_{\mu_2}$ & 0.02 & 0.02 & 0.02 & 0.21 & 0.60 & 0.32 & 0.65 & 0.38 & 0.70 & 0.48 & 1.02 & 0.75 & 1.86 & 1.39 & 1.88 & 1.31 & 1.80 & 2.54 & 
%KL & 0.99 \\ 
%$\epsilon_{\sigma_2}$ & 1.00 & 1.16 & 0.60 & 0.37 & 0.36 & 0.18 & 0.31 & 0.18 & 0.27 & 0.10 & 0.19 & 0.05 & 0.12 & 0.09 & 0.09 & 0.07 & 0.10 & 0.10 & 
%KL' & 0.99\\
%\hline
\hline
\multicolumn{21}{|c|}{LReLU, Dropout (0.2)}\\
\hline
$\epsilon_{\mu_1}$ & - & 0.01 & 0.02 & 0.04 & 0.12 & 0.07 & 0.13 & 0.09 & 0.17 & 0.11 & 0.21 & 0.14 & 0.26 & 0.14 & 0.27 & 0.16 & 0.23 & 0.32 & 
KL & 0.05 \\
\hline
$\epsilon_{\mu_2}$ & - & 0.01 & 0.02 & 0.01 & 0.02 & 0.02 & 0.03 & 0.02 & 0.04 & 0.03 & 0.05 & 0.04 & 0.08 & 0.05 & 0.08 & 0.07 & 0.10 & 0.15 & 
KL & 0.02 \\ 
$\epsilon_{\sigma_2}$ & - & 1.00 & 1.00 & 1.02 & 0.93 & 0.95 & 0.90 & 0.92 & 0.81 & 0.79 & 0.73 & 0.69 & 0.71 & 0.69 & 0.71 & 0.65 & 0.56 & 0.69 & 
KL' & 0.02\\
\hline
\multicolumn{21}{|c|}{Logistic Bernoulli}\\
\hline
$\epsilon_{\mu_1}$ & - & 0.02 & 0.02 & 0.04 & 0.12 & 0.10 & 0.18 & 0.15 & 0.29 & 0.21 & 0.37 & 0.25 & 0.48 & 0.31 & 0.44 & 0.33 & 0.27 & 0.47 & 
KL & 3.46 \\
\hline
$\epsilon_{\mu_2}$ & - & 0.02 & 0.02 & 0.03 & 0.04 & 0.04 & 0.05 & 0.06 & 0.12 & 0.10 & 0.20 & 0.14 & 0.30 & 0.19 & 0.28 & 0.19 & 0.14 & 0.25 & 
KL & 0.21 \\ 
$\epsilon_{\sigma_2}$ & - & 1.12 & 1.00 & 1.02 & 0.92 & 0.99 & 0.78 & 0.97 & 0.59 & 0.95 & 0.56 & 0.94 & 0.61 & 0.94 & 0.73 & 0.96 & 0.74 & 0.99 & 
KL' & 0.1\\
\hline
\multicolumn{21}{|c|}{Logistic Transform, Noisy input $\N(0,0.01)$}\\
\hline
$\epsilon_{\mu_1}$ & 0.02 & 0.17 & 0.17 & 0.28 & 0.49 & 0.53 & 0.98 & 0.96 & 1.47 & 1.39 & 1.62 & 1.52 & 2.13 & 2.06 & 2.36 & 2.31 & 1.86 & 2.55 & 
KL & 3.46 \\
\hline
$\epsilon_{\mu_2}$ & 0.02 & 0.03 & 0.03 & 0.16 & 0.40 & 0.43 & 0.82 & 0.81 & 1.26 & 1.21 & 1.39 & 1.32 & 1.88 & 1.84 & 2.05 & 2.03 & 1.60 & 2.20 & 
KL & 2.68 \\ 
$\epsilon_{\sigma_2}$ & 1.00 & 1.01 & 0.40 & 0.41 & 0.28 & 0.29 & 0.23 & 0.23 & 0.12 & 0.11 & 0.05 & 0.05 & 0.02 & 0.02 & 0.01 & 0.01 & 0.01 & 0.01 & 
KL' & 2.68 \\
\hline
\end{tabular}
}
\caption{\label{tab:accuracy-CIFAR}
Accuracy of approximation for CIFAR-10 Model. The notation is the same as in~\cref{tab:accuraacy-LeNet}, the row ``P'' in the head of the networks is the average pooling layer.
}
\end{table}
\begin{table}[t]
\centering
\setlength{\tabcolsep}{3pt}
\resizebox{0.85\linewidth}{!}{
\begin{tabular}{cc}
\begin{tabular}{|c|cc|cc|cc|c|}
\multicolumn{8}{c}{LReLU}\\
\hline
%Metric $\backslash$ Layers
& C%
& A%
& C%
& A%
& C%
& A%
& F\\%
\hline
\multicolumn{8}{|c|}{Init=None}\\
\hline
$\mu$ & 0.10 & 0.27 & 0.23 & 0.22 & 0.39 & 0.33 & 0.41\\
$\sigma$          & 0.89 & 1.32 & 0.85 & 1.01 & 0.92 & 1.03 & 0.93\\
\hline
\multicolumn{8}{|c|}{Init=BN}\\
\hline
$\mu$ & 0.10 & 0.49 & 0.13 & 0.36 & 0.29 & 0.52 & 0.49\\
$\sigma$ & 0.89 & 1.20 & 1.04 & 1.22 & 1.28 & 1.47 & 1.49\\
\hline
\multicolumn{8}{|c|}{Init=AP2}\\
\hline
$\mu$    & 0.10 & 0.20 & 0.26 & 0.25 & 0.38 & 0.37 & 0.38\\
$\sigma$ & 0.89 & 1.23 & 0.92 & 1.17 & 1.08 & 1.37 & 1.21\\
\hline
\end{tabular}
&
\begin{tabular}{|c|cc|cc|cc|c|}
\multicolumn{8}{c}{Logistic Transform}\\
\hline
%Metric $\backslash$ Layers
& C%
& A%
& C%
& A%
& C%
& A%
& F\\%
\hline
\multicolumn{8}{|c|}{Init=None}\\
\hline
$\mu$             & 0.10 & 0.09 & 0.24 & 0.24 & 0.36 & 0.36 & 0.27\\
$\sigma$          & 0.89 & 0.91 & 0.76 & 0.76 & 0.76 & 0.77 & 0.75\\
\hline
\multicolumn{8}{|c|}{Init=BN}\\
\hline
$\mu$ & 0.10 & 0.04 & 0.20 & 0.19 & 0.34 & 0.34 & 0.25\\
$\sigma$          & 0.93 & 1.06 & 1.02 & 1.10 & 1.15 & 1.16 & 1.16\\
\hline
\multicolumn{8}{|c|}{Init=AP2}\\
\hline
$\mu$             & 0.10 & 0.04 & 0.17 & 0.15 & 0.37 & 0.37 & 0.28\\
$\sigma$          & 0.89 & 1.15 & 1.02 & 1.13 & 1.16 & 1.22 & 1.21\\
\hline
\end{tabular}
\end{tabular}
}
\caption{\label{tab:data-acc-LeNet}
Accuracy of dataset statistics for LeNet/MNIST, estimated by AP2 model with a single-pixel input $(\mu_0, \sigma_0^2)$. 
The notation is the same as in~\cref{tab:accuraacy-LeNet}.
}
\end{table}
%
%
%
%%==========================================================================================
\paragraph{Dataset Statistics and Analytic Normalization}
\cref{tab:data-acc-LeNet} shows the accuracy of estimating neuron statistics over the dataset using the proposed technique.
In convolutional networks, the task is to estimate the mean and variance $\mu^*, \sigma^*$ per {\em channel} in each layer, \ie, the statistics are over the input dataset and the spatial dimensions. With our method, the estimates are computed by propagating through the network the statistics of the input dataset $\mu^*_0, \sigma^*_0$ (that obviously do not depend on the network). The propagation works directly with spatial averages, the batch dimension and the spatial dimensions are not used and the only relevant dimension is channels (see details of this efficient implementation in~\cite{shekhovtsov-18-norm}).
The reported errors in estimating the statistics are averaged over the channels.
We study these errors for three cases: randomly initialized networks, networks re-initialized with batch normalization (BN)~\cite{IoffeS15} as described \eg in~\cite{Salimans2016WeightNA} and our analytic normalization~\cite{shekhovtsov-18-norm}. The re-initialization consists of recurrently going through the layers, applying the normalization and estimating the statistics of the next layer. It is clearly seen that the accuracy of the analytic normalization is completely sufficient for the purpose of network initialization and normalization~\cite{shekhovtsov-18-norm} (i.e.~the true variance is close to one and the deviation from the true mean is less than the true standard deviation). This normalization is computationally cheap, continuously differentiable and is applicable to training of standard networks as well as variance-propagating networks. In comparison to BN, it however lacks additional generalization properties~\cite{shekhovtsov-18-norm}.
%
%
%%==========================================================================================
\paragraph{Analytic Dropout}
In this experiment we demonstrate the utility of using our approximation during training. \cref{fig:dropout-CIFAR} shows a comparison of plain training, dropout~\cite{srivastava14a}, which samples multiplicative Bernoulli noise during training, and analytic dropout, in which our propagation is used. The dropout layers are applied after every activation and there is no input dropout.
All methods start from the same BN-initialized point and use the same learning rate and schedule ($0.001\cdot 0.96^k$ in epoch $k$) and the same optimizer (Adam~\cite{KingmaB14}). We intentionally do not use batch normalization during training, since it has a regularization effect of its own~\cite{IoffeS15} that needs to be studied separately. With the analytic propagation we show two results: initialized the same way as the baseline and using~\cite{shekhovtsov-18-norm}.
%
%By no means this experiment is a comprehensive comparison. Testing different learning schedules and optimizers may be needed to get the full picture, but it shows that the proposed technique has a potential worth a further development.
This comparison is limited but it shows that AP2 propagation significantly improves validation error in a deep network while not slowing the training down, which qualitatively agrees with the findings of~\cite{wang2013fast} for smaller networks.
\begin{figure}[b]
\centering
\setlength{\tabcolsep}{0pt}
\begin{tabular}{cc}
\includegraphics[width=0.49\linewidth]{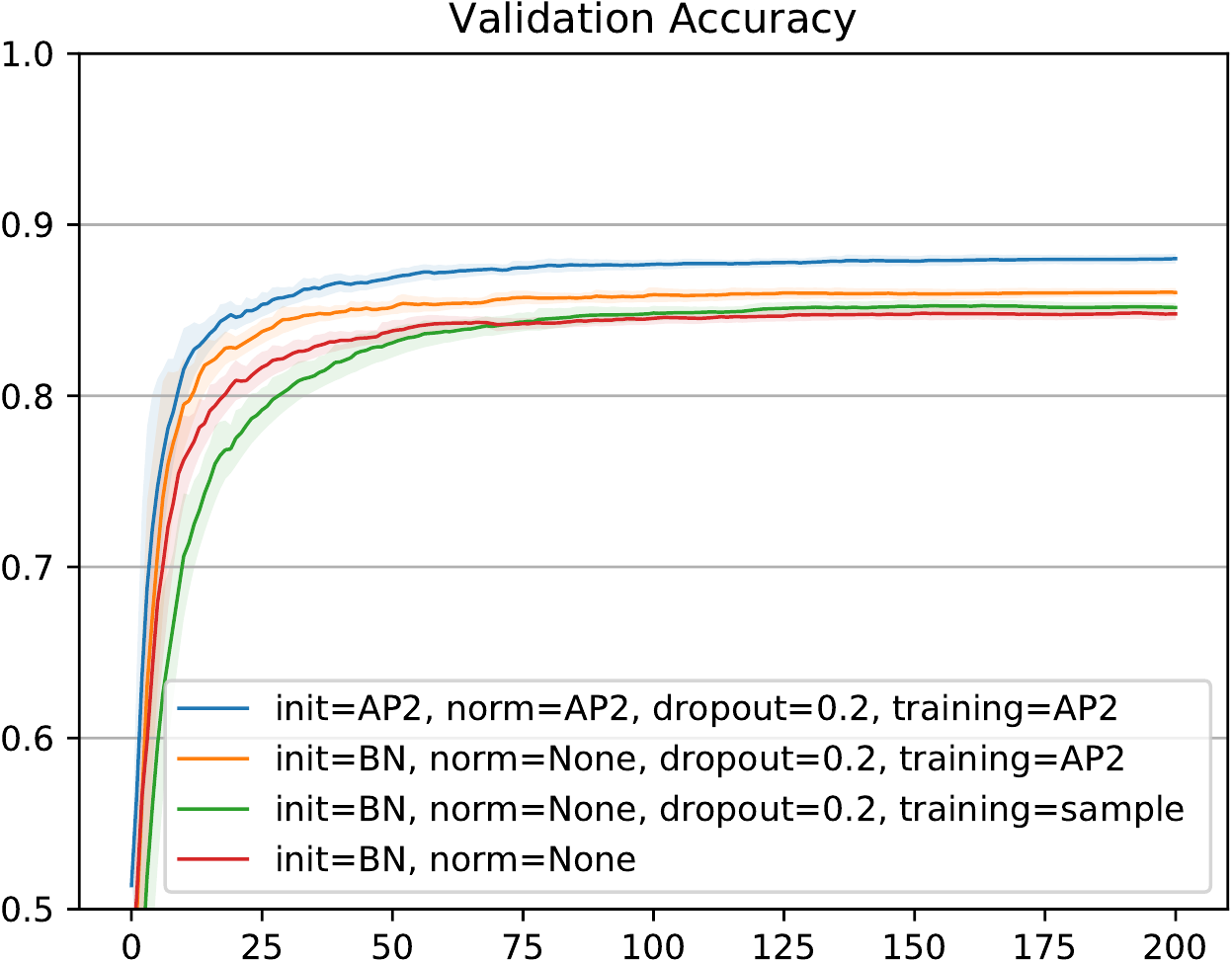}&
~\includegraphics[width=0.49\linewidth]{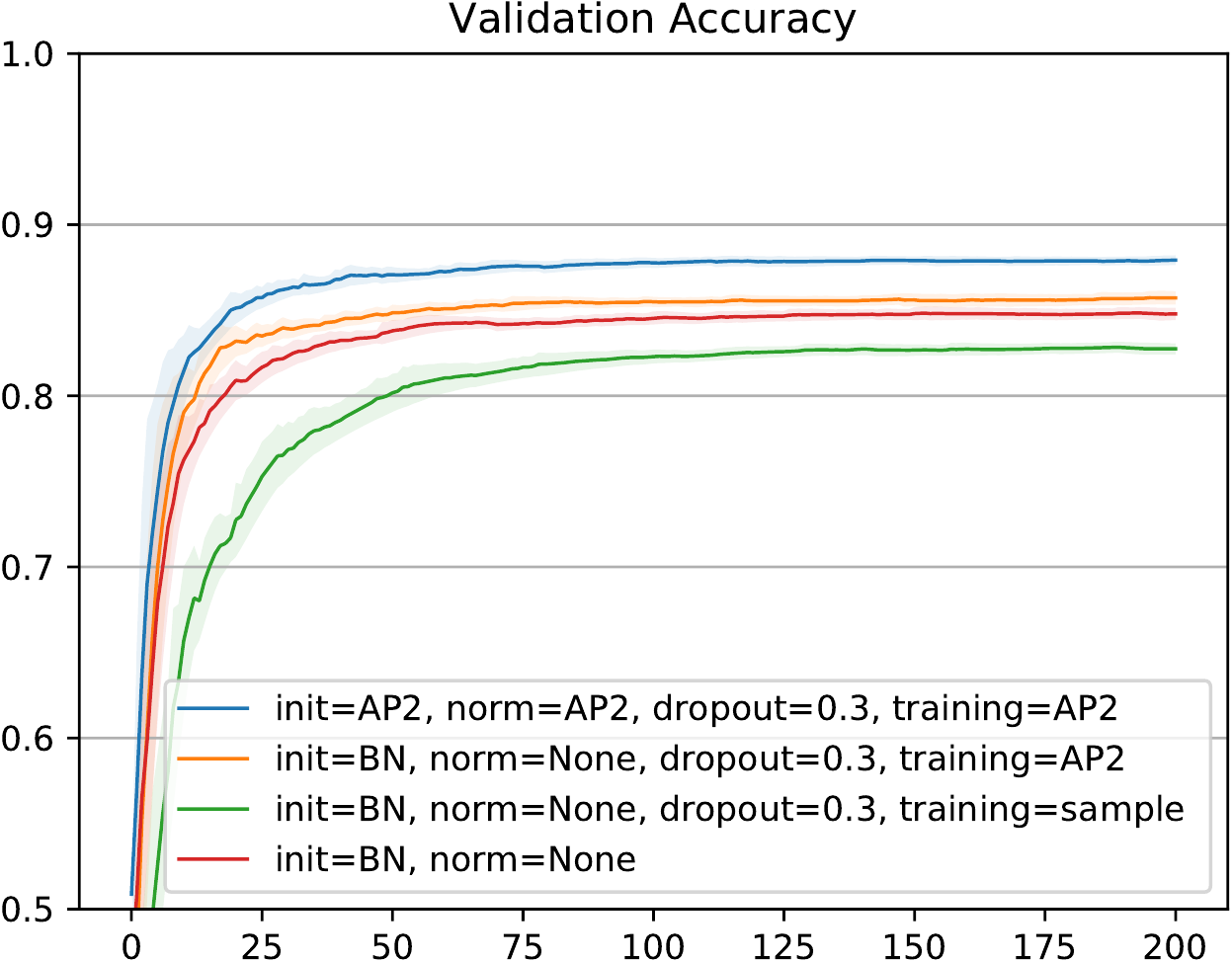}
\end{tabular}
\caption{\label{fig:dropout-CIFAR}
CIFAR-10: Comparison of plain training, dropout and analytic dropout for drop probabilities 0.2 and 0.3. For the analytic dropout (training=AP2) we tested two variants: with the same initialization as the baseline and with AP2 initialization and normalization.
The plots show validation accuracies during training epochs (running average). Notice that dropout (green) fails to improve the validation accuracy compared to standard training (red) for higher drop probability within the training schedule. The proposed analytic dropout clearly improves the validation accuracy without a noticeable slow down.
%
%Left: training loss (solid) and validation loss (dashed) vs training epochs. The values are adaptively averaged over the epochs. For validation losses the shaded area shows 1 standard deviation interval.
%Notice that dropout (green) fails to improve the validation loss compared to standard training (red) and actually performs worse in accuracy. Analytic dropout (orange, blue) improves validation loss and accuracy.
}
\end{figure}
%%==========================================================================================
\paragraph{Stability}
%\revisit
%fig:stability-MLP
We made a conjecture that propagating uncertainty may improve stability of the predictions under noise or adversarial attacks. 
The idea can be demonstrated on a simple NN with one hidden layer of 100 units. This simple model trained on the MNIST dataset reveals quite surprising results. We compared training of a standard NN with sigmoid activations and a model with Bernoulli-logistic activations trained using AP2, assuming input noise with variance $0.1$. The results in~\cref{fig:stability-MLP} show that the latter model, while achieving the same test accuracy, is significantly more stable to Gaussian noise. The same dependance is observed for gradient sign attack~\cite{Goodfellow-15-adversarial} shown in~\cref{fig:stability-MLP-adv}.
The shown stability results do not immediately scale to deep networks. We see two reasons for this. First, the approximation quality of propagating input uncertainty degrades with depth as we have seen above. Second, the stability depends on both the propagation method and the choice of parameters. While propagating variance can deliver a more stable classifier, the choice of parameters is still crucial. In particular, parameters may exist such that the network posterior is always deterministic regardless of the input uncertainty. In this case, propagating the variance is useless. %Unfortunately, when all training cases can be correctly classified, the maximum likelihood learning tends to boots such discriminative parameters. 
These issues need to be addressed in the future work.
\begin{figure}[t]
\centering
\includegraphics[width=0.49\linewidth]{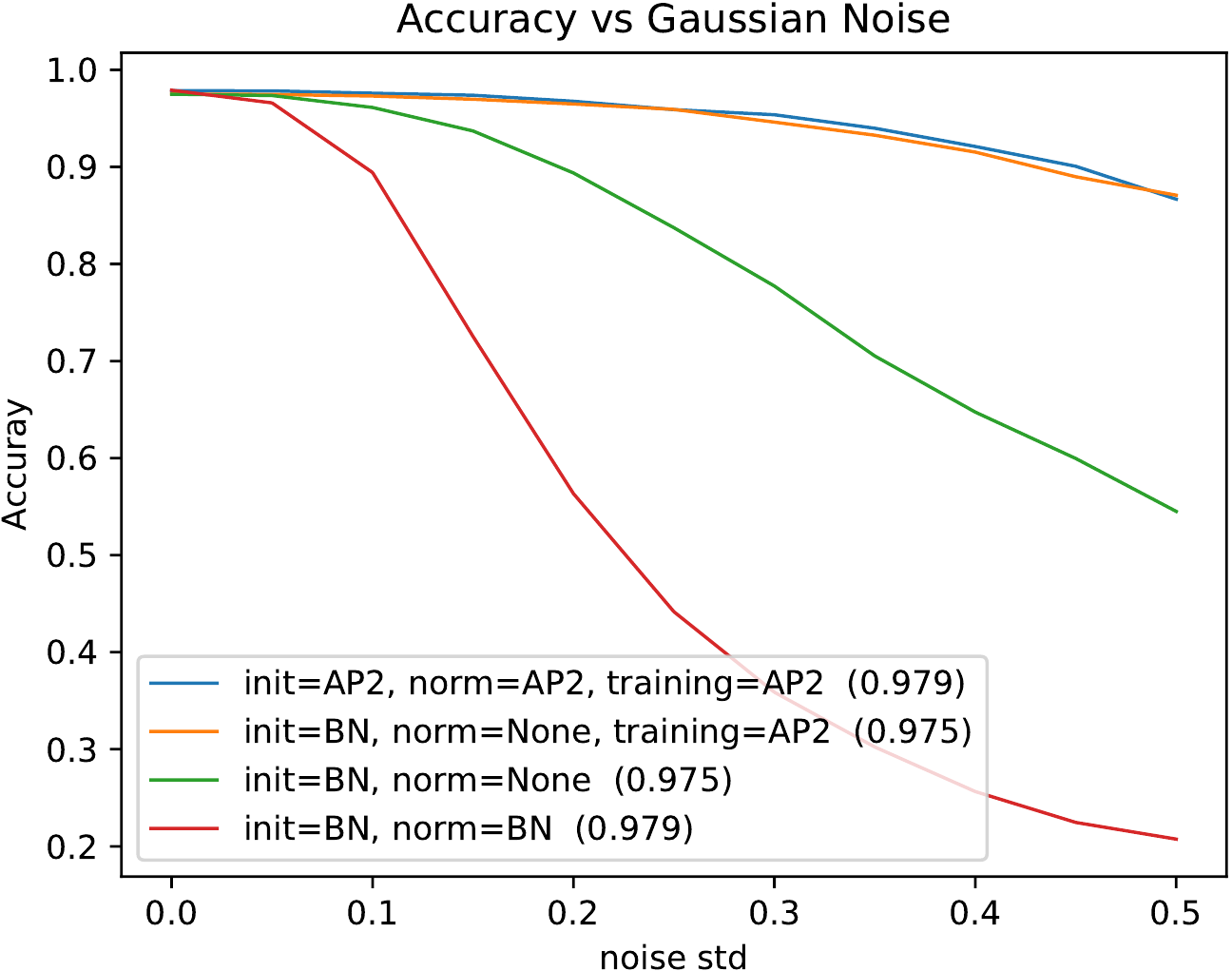}\ %
\includegraphics[width=0.49\linewidth]{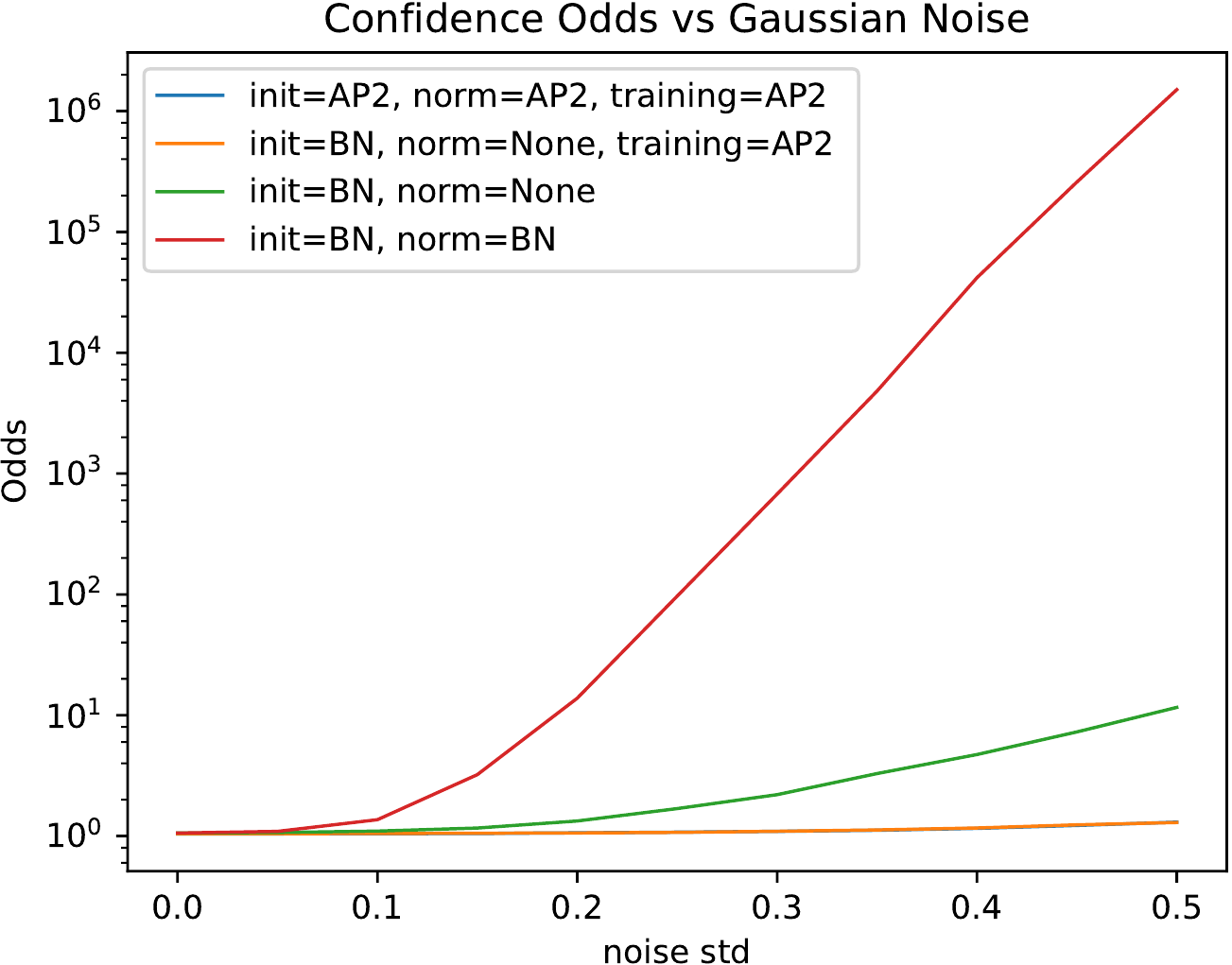}%
\caption{\label{fig:stability-MLP}
Stability evaluation of MLP/MNIST model under Gaussian noise.
All networks have similar test accuracy in the noise-free scenario (numbers in brackets). Bernoulli model with analytic propagation shows a better stability \wrt random perturbations of the input. Batch normalization achieves a lower training loss than the standard NN which leads to a significant drop in the stability. When the network trained with BN makes errors under noise it does so with a very high confidence odds -- the ratio of the predicted class probability over the true one: $\max_y q(y) / q(y^*)$.
}
\end{figure}

%In particular, when the dataset is separable (\ie, such parameters exist that all training examples can be classified correctly), 

%% file: tex/appendix.tex
\newpage
\appendix
\numberwithin{figure}{section}
\addtocontents{toc}{\protect\setcounter{tocdepth}{2}}
\pagestyle{plain}
%
% Reset counters
\setcounter{figure}{0}
\setcounter{table}{0}
\counterwithin{figure}{section}
\counterwithin{table}{section}
\counterwithin{theorem}{section}
\counterwithin{proposition}{section}
\counterwithin{lemma}{section}
%
%\title{\mytitle (Appendix)}
%\author{Paper ID \SubNumber \\[10pt]}
\title{Appendix}
\author{}
\institute{}
%\authorrunning{}
%\titlerunning{ECCV-18 submission ID \SubNumber}
%\authorrunning{ECCV-18 submission ID \SubNumber}
\titlerunning{\,}
\authorrunning{\,}
\maketitle
%
%%\appendices
%\appendix
%\addtocontents{toc}{\protect\setcounter{tocdepth}{2}}
%\pagestyle{plain}
%
%% Reset counters
%\setcounter{figure}{0}
%\setcounter{table}{0}
%\counterwithin{figure}{section}
%\counterwithin{table}{section}
%
%\twocolumn[{%
 %\centering
 %\LARGE \mytitle \\ %Supplementary Material \\[1.5em] 
 %(CVPR Submission \#\cvprPaperID\ Supplementary Material) \\[1.5em]
 %\normalsize
%}]
%
% TOC
% remove dots and page numbers
\let\Contentsline\contentsline
\renewcommand\contentsline[3]{\Contentsline{#1}{#2}{}}
\makeatletter
\renewcommand{\@dotsep}{10000}
\makeatother
\def\authcount#1{}
\tableofcontents
%
%===========================================================================================
\section{Used Facts on Approximate Marginalization}\label{sec:proofs}
%%%%%%%%%%%% Stochastic
%%%%%%%%%%%% Marginals
\begin{lemma}\label{PMarginals} Let $X$ be a \rv with components $X_i$ and pdf $p(X)$. The closest factorized approximation $q(X) = \prod_i q(X_i)$ to $p(X)$ in terms of forward KL divergence is given by the marginals $p(x_i) = \sum_{(x_j \mid j\neq i)}p(x)$.
\end{lemma}
\begin{proof}
%Let us find a factorized distribution $q ( y \mid I)$ minimizing the KL divergence to the posterior $p(y \mid I)$. Omitting ``$|I$'' for clarity, 
%The best approximation $q(y)$ should minimize the forward KL divergence
Minimizing 
\begin{align}
KL(p(X) \| q(X) ) = \E_{p(X)} \log \frac{p(X)}{q(X)},
\end{align}
in $q$ is equivalent to maximizing
\begin{align}\label{crossent}
\E_{p(X)} \log {q(X)}.
\end{align}
Assuming $q(X) = \prod_i q(X_i)$,
the negative cross-entropy above expresses as
\begin{align}
 \sum_{x} p(x)\sum_{i}\log q(x_i) = \sum_{i} \sum_{x_i} p(x_i) \log q(x_i),
\end{align}
which is maximum when $q(x_i) = p(x_i)$.
\end{proof}

\begin{lemma}\label{GMarginals}
Let $X$ be a continuous \rv. The closest approximation to $p(X)$ by a Gaussian $q(x) = \frac{1}{\sqrt{2\pi}\sigma}\exp(-\frac{(x-\mu)^2}{2\sigma^2})$ in forward KL divergence is given by moment matching: $\mu = \E[X]$, $\sigma^2 = \E[X_2-\mu]$.
\end{lemma}
\begin{proof}
This is essentially the same as in maximum likelihood estimate of normal distribution. Differentiating~\eqref{crossent} and solving for the critical point we get for the mean
\begin{align}
 0 = E_{p(X)} \frac{\partial}{\partial \mu}\log q(X)  = \E_{p(X)} \frac{(X-\mu)}{\sigma^2},
\end{align}
from which $\mu = \E_{p(X)}[X]$. And for the variance:
\begin{align}\notag
 0 = E_{p(X)} \frac{\partial}{\partial \sigma}\log q(X)  = \E_{p(X)} \Big(-\frac{1}{\sigma} + \frac{(X-\mu)^2}{\sigma^3}\Big),
\end{align}
from which $\sigma^2 = \E_{p(X)}[(X-\mu)^2]$.
\end{proof}
%
%==============================================================================
\section{Auxiliary Results}
%==============================================================================
\subsection{From Latent Variable Model to Belief Network}\label{sec:cond-densities}
\begin{lemma}[Stochastic binary neuron]\label{CStochasticBN}
Let $Y = \leftbb X - Z \geq 0 \rightbb$, where $Z$ is independent \rv with cdf $F_Z$. Then $\Pr(Y{=}1 \mid X) = F_Z(X)$.
\end{lemma}
% \begin{proof}
% We have
% \begin{align}
% %& F_{Y\mid X}(y) = 
% & \Pr\{ Y\,{\leq}\,y \mid X\,{=}\,x\} = 
% \E_Z \big[ \leftbb x - Z \geq 0 \rightbb  \leq y \big] \\
% & \notag
% = \int_{-\infty}^{\infty} p_Z(z) \Big( \leftbb x < z \wedge 0 \leq y \rightbb + \leftbb x \geq z \wedge 1 \leq y \rightbb \Big) dz
% \end{align}
% \begin{align}
% \notag 
% & = \int_{x}^{\infty} p_Z(z) dz \,\leftbb y\geq 0 \rightbb + \int_{-\infty}^{x} p(z) dz \,\leftbb y \geq 1 \rightbb\\
% &
% = (1-F_Z(x))\leftbb y\geq 0 \rightbb + F_Z(x) \leftbb y \geq 1 \rightbb
% \end{align}
% \begin{align}
% & 
% =
% \begin{cases}
% 0, & y < 0,\\
% 1 - F_Z(x), & y \in [0,1),\\
% 1, & y \geq 1.
% \end{cases}
% \end{align}
% Because $Y$ attains only two values $0$ and $1$, its distribution is discrete and the probability mass
% $p(Y{=}1 \mid X\,{=}\,x)$ expresses as
% \begin{align}
% F_{Y \mid X{=}x}(1) - \lim\limits_{y \rightarrow 1^{-}} F_{Y \mid X{=}x}(y)
% %d F_{Y \mid X\,{=}\,x}(y) |_{y=1} 
% = F_Z(x).
% \end{align}
% \end{proof}
%
\begin{proof}
 We have
 \begin{equation}
  \Pr(Y{=}1 \mid X) =  \E \big[ \leftbb X - Z \geq 0 \rightbb \bigm | X \big] = 
  \E \big[ \leftbb Z \leq X \rightbb \bigm | X \big] = F_Z(X) .
 \end{equation}
% Hence $p(Y{=}1 \mid X{=}x) = F_Z(x)$.
\end{proof}

Thus, a stochastic binary neuron with logistic noise is equivalent to a Bernoulli-logistic unit~\protect{\cite[p.4]{Williams1992}} (this is a well-known interpretation in logistic regression) and such networks are equivalent to logistic belief nets~\cite{Neal:1992}. %, as was also observed in~\protect{\cite[p.4]{Williams1992}}. %but also is well-known for logistic regression.
\begin{lemma}[Stochastic Rectifier]\label{CStochasticReLU}
Let $Y = \max(0,X - Z)$, where $Z$ is an independent \rv with cdf $F_Z$ and pdf $p_Z$. Then $p(Y{=}y \mid X) = (1-F_Z(X-y)) \delta_0(y) + p_Z(X-y) \leftbb y{>}0 \rightbb$, where $\delta$ is the Dirac distribution.
\end{lemma}
% \begin{proof} We have
% \begin{align}
% &\Pr\{ Y\,{\leq}\,y \mid X\,{=}\,x\} = \E_Z \leftbb \max(0, x - Z) \leq y \rightbb \\
% & \notag
% = \E_Z \leftbb x < Z \wedge y \geq 0 \rightbb  + \E_Z \leftbb x  \geq Z \wedge y \geq x - Z \rightbb\\
% & \notag
% = \Big( \int_{x}^{\infty} p_Z(z) dz + \int_{x-y}^{x} p(z) dz \Big) \leftbb y\geq 0 \rightbb \\
% &
% = \Big( 1-F_Z(x) + F_Z(x)  - F_Z(x-y) \Big) \leftbb y\geq 0 \rightbb\\
% & 
% =
% \begin{cases}
% 0, & y < 0,\\
% 1 - F_Z(x-y), & y \geq 0.
% \end{cases}
% \end{align}
% This distribution has a discrete component at $0$.
% Respectively, we can write the density using distributions as:
% $p(y \mid x) = (1-F_Z(x-y)) \delta_0(y) + p_Z(x-y) \leftbb y > 0 \rightbb$.
% \end{proof}
\begin{proof}
 Let us begin with deriving the conditional cdf.~of $Y$. We have
 \begin{align}
  \Pr\{ Y\,{\leq}\,y \mid X \} & = 
  \E \big[\leftbb \max(0, X - Z) \leq y \rightbb \bigm | X \big] \\ 
   & = \E \big[\leftbb \min(X,Z) \geq X - y \rightbb \bigm | X \big]
 \end{align}
 The indicator function is obviously zero if $y < 0$. If, on the other hand, $y \geq 0$ then the indicator function is nonzero only if $Z \leq X$ and we get
 \begin{equation}
  \Pr\{ Y\,{\leq}\,y \mid X \} = 
  \begin{cases}
   0 & \text{if $y < 0$,} \\
   \E \big[\leftbb Z \geq X - y \rightbb \bigm | X \big] & \text{otherwise.} 
  \end{cases}
 \end{equation}
 Hence
 \begin{equation}
  \Pr\{ Y\,{\leq}\,y \mid X = x \} = 
  \begin{cases}
   0 & \text{if $y < 0$,} \\
   1 - F_Z(x-y) & \text{otherwise.} 
  \end{cases}
 \end{equation}
This distribution has a discrete component at $0$.
Consequently, we can write the density using the $\delta$ distribution as:
$p(y \mid x) = (1-F_Z(x-y)) \delta_0(y) + p_Z(x-y) \leftbb y > 0 \rightbb$.
\end{proof}

%
%
%==========================================================================================
\subsection{Gaussian Belief Network View}\label{sec:NLGBN}
This subsection establishes more connections to related work. The latent variable model~\eqref{LVM} can be equivalently represented as a belief network of noisy activations $\hat X^k = W^k X^{k-1} + Z^{k}$ as primary variables. They become connected by the conditional densities
\begin{align}\label{view-NLGBN}
p(\hat X^k \mid \hat X^{k-1}) = \phi^k(\hat X^k_j - W^k f^{k-1}(\hat X^{k-1}) ),
\end{align}
where $\phi^k$ is the pdf of the noise $Z^k$, $\hat X^0 = X^0$ and $f^0$ is identity. The original variables are recovered as deterministic mappings: $X^k = f^k(\hat X^k)$. Note that regardless whether the original variables $X^k$ were binary or real valued, the noisy activations are always real-valued and are connected by a conditional pdf of a simple form. This representation proposed in~\cite{Frey-99} with Gaussian noise is known as {\em Nonlinear Gaussian Belief Network} (NLGBN). Conversions between representations~\eqref{DCIM} and~\eqref{view-NLGBN} and the effect of this choice on different algorithms was studied in~\cite{Kingma:2014,papaspiliopoulos2007,Kingma13-fast}.

There are therefore at least 3 views on the model: the latent variable model~\eqref{LVM}, the belief network model~\eqref{DCIM} and the belief network of noisy activations~\eqref{view-NLGBN}. Not every belief network given by~\eqref{DCIM} can be represented using the other views, but there is a rich family that can be represented in the form~\eqref{LVM} and equivalently transformed to others.

%Our approximate inference utilizes only two moments of latent variables noise noise. The assumption we made in~\cref{ff-inference} to approximate the posterior probabilities of continuous variables by Gaussian distributions, was not strictly speaking necessary. Only the approximation of a sum of a linear combination of continuous variables and the latent noise in~\cref{prop:gaussian-input} was important.
%However, since the approximation is concerned only with the mean and variance, there is no advantage of considering latent variables 
Our approximate inference utilizes only two moments of latent variables noise. %Therefore, there is no advantage in considering other noise distributions than Gaussian.
It can be fairly assumed that all latent variables are Gaussian, therefore the model we consider is not much different from NLGBN. Our inference method can be derived from this model by assuming Gaussian factorized posteriors of all activations $\hat X^k$ and propagating their moments. The inference in~\cite{Frey-99} is based on the variational lower bound formulation, where the approximate posterior $q$ should minimize the backward KL divergence $KL(q\|p)$.
%
%==========================================================================================
\subsection{Parameters Setting in~\cref{example:and}}\label{AND-params}
Parameters $a$ and $b$ are chosen such that $\E_X \S (a(X_1+X_2) + b) \geq 1-\varepsilon$ holds for $p(X_1{=}1)=1$, $p(X_2{=}1)=1$ and $\E_X \S (a(X_1+X_2) + b) \leq \varepsilon$ holds for $p(X_1{=}1)=1$, $p(X_2{=}1)=0$. In these cases the expectation is trivial. We have
\begin{align}
&1/(1+e^{-(2a+b)}) \geq 1-\varepsilon,\\
&1/(1+e^{-(a+b)}) \leq \varepsilon,
\end{align}
	from which we get
\begin{align}
2a+b \geq \log(\frac{1-\varepsilon}{\varepsilon}),\\
 a+b \leq -\log(\frac{1-\varepsilon}{\varepsilon}),
\end{align}
and subsequently
\begin{align}
a  = 2 \log(\frac{1-\varepsilon}{\varepsilon}) = 5.89,\\
b  = -3 \log(\frac{1-\varepsilon}{\varepsilon}) = -\frac{3}{2}.
\end{align}
%
%
%==============================================================================
\section{Details of Approximations}\label{sec:approximations}
%==============================================================================
\subsection{Summary List of Approximations}\label{sec:approx-list}
Below we list approximations for propagating moments through common layers. Functions $\phi$ and $\Phi$ denote respectively the pdf and the cdf of a standard normal distribution.
%
\input{tex/table_approx.tex}
%==============================================================================
\subsection{Heaviside Step Function}\label{detail:Heaviside}
We need to approximate the mean $\mu'$ of the indicator $\leftbb X \geq 0 \rightbb$. Assuming that $X$ has normal distribution with mean $\mu$ and variance $\sigma^2$, we have:
\begin{align}
\mu' &= \int_{0}^{\infty} \phi((x-\mu)/\sigma) \d x = \int_{-\mu/\sigma}^{\infty} \phi(x) \d x \\
& = 1 - \Phi (-\mu /\sigma) = \Phi (\mu/\sigma).
\end{align}
Since the square of the indicator matches itself, the second moment is also $\mu'$. The variance therefore equals $\sigma'^2 = \mu' - \mu'^2 = \mu'(1-\mu')$.
\par
Assuming $X$ has logistic distribution with mean $\mu$ and variance $\sigma^2$, we have 
\begin{align}
\mu' =  \int_{0}^{\infty} \S'((x-\mu)/s) \d x = \S (\mu /s),
\end{align}
where $s = \sigma/\sigma_\S$.
%
%==============================================================================
\subsection{Mean of the Logistic Trasform / Bernoulli Unit}\label{sec:pae}
\paragraph{Piecewise Exponential Approximation}
An approximation for estimating $\E_X \S(w\T X)$, more accurate than substituting the mean, was proposed in~\cite{AstudilloN11}. The function $\S$ is approximated as the following piecewise exponential function:
\begin{align}\label{PEA-sigmoid}
\S(z) \approx \L(z) = \begin{cases}
	2^{z-1},& \IF z < 0\\
	1 - 2^{-z-1},& \IF z \geq 0,
	\end{cases}
\end{align}
shown in~\cref{fig:pea}. 
Assuming $w\T X$ is normally distributed with mean $\mu$ and variance $\sigma^2$, authors of~\cite{AstudilloN11} obtain the expression $\E_X \L(w\T X) = $
\begin{align}\label{PEA}\tag{PEA}
%&\E_X \L(w\T X) = 
\Phi(\mu / \sigma) + 2^{\frac{\ln(2)}{2}\sigma^2-1}\Big(
2^{\mu} \Phi\big((-\mu{-}\ln(2)\sigma^2)/ \sigma\big)
- 2^{-\mu} \Phi\big((\mu{-}\ln(2)\sigma^2)/\sigma \big) \Big).
%\tag{PEA}
\end{align}
%where $\Phi$ is the cdf of the standard normal distribution.
The error of the approximation comes from two sources: the approximation~\eqref{PEA-sigmoid} and the assumption of normality of $w\T X$.
It is easy to see that $\L$ is actually the Laplace distribution with scale $s = 1/\log(2)$. Thus,~\cite{AstudilloN11} propagate uncertainty through the Laplace cdf activation assuming the input is normally distributed.
\begin{figure}[h]
\includegraphics[width=0.5\linewidth]{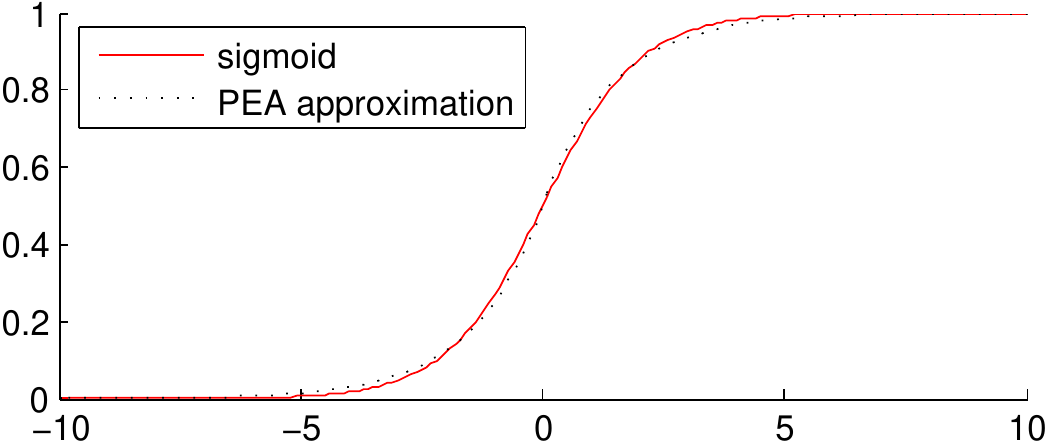}
\caption{Piecewise exponential approximation~\cite{AstudilloN11} of logistic cdf, equal to the Laplace cdf with scale $1/\log 2$.}\label{fig:pea}
\end{figure}
%While the error functions in~\eqref{PEA} 
%This is not directly usable in practice because of the error functions and needs to be approximated further.

We evaluated different approximations AP1, \cref{AP2a},\cref{AP2b} and \cref{PEA} in \cref{fig:approx-eval} by measuring the forward KL divergence to the true posterior distribution (computed by convolving densities). We compare to the sampling-based approximation. %and piece-wise exponential approximation (PEA)~\cite{AstudilloN11}. The latter assumes that $w\T X$ has a normal distribution and replaces the logistic cdf $\S$ with the Laplace cdf (details~\cref{sec:pae}). 
The results indicate that the baseline AP1 can be significantly improved and that our approximations are on par with sampling.
\subsection{Variance of Logistic Transform}\label{detail:logistict-var}
For the approximation of the mean we can use the same expression as for Logistic Bernoulli, \eg, recall~\eqref{AP2b} is $\mu' = \S(\mu/s)$, where $s = (\sigma^2 / \sigma^2_\S + 1)^{\frac{1}{2}}$.

The approximation of the variance is more involved in this case. There is no tractable analytic expression for the second moment assuming either normal or logistic distribution of $X$. The approximation of variance based on PEA~\cite[eq.14]{AstudilloN11} was found inaccurate\footnote{
The mentioned numerical accuracy is of the variance expression computed with Mathematica based on  PEA approximation of the logistic function.
The equation 14 from \cite{AstudilloN11} was giving different results, inaccurate even for sigma around $1$, possibly due to a mistake in the equation.} for small $\sigma$ and because it makes a heavy use of error functions (that can only be approximated with series) is numerically unstable for a wider range of parameters.

\par
\paragraph{Practical Approximation}
We have constructed the following practical approximation:
\begin{align}\label{logistic-heuristic-approx}
\sigma'^2 \approx {\textstyle 4} (1+4 \sigma^{-2})^{-1} (\mu'(1-\mu'))^2.
\end{align}
%The factor ${\textstyle \frac{1}{4}} (1+4 \sigma^{-2})^{-1}$ is chosen by approximating the dependence for $\mu=0$. 
It is set up to match the following asymptotes. Note that for $\mu=0$, there holds $(\mu'(1-\mu'))^2 = 1/16$ for all $\sigma$. Then for $\mu=0$ and $\sigma^2\rightarrow \infty$ it must be $\sigma'^2={\textstyle \frac{1}{4}}$: we can think of this limit as rescaling the sigmoid function, which approaches then the Heaviside step function, which has variance $\frac{1}{4}$ at $\mu=0$. Another asymptote is for $\sigma^2\rightarrow 0$, where the variance should approach that of the linearized model with the slope given by the derivative of the logistic function at $\mu=0$. This is satisfied since $\mu'$ approaches $\S(\mu)$ and the derivative of the logistic function can be written as $\S(\mu)(1-\S(\mu))$. The variance is then proportional to the square of the derivative. Thus, the model~\eqref{logistic-heuristic-approx} is designed to be accurate for small $\sigma$. 
%We should have then $\sigma'^2={\textstyle \frac{1}{16}} \sigma^2$ in this limit, which is satisfied by~\eqref{logistic-heuristic-approx}.
%
For $\mu=0$ the maximum relative error in $\sigma'$ is $14\%$ for the whole range of $\sigma$ (numerical simulations). For $\sigma\leq 1$ the maximum relative error is $26\%$ for all $\mu$, more accurate for smaller $\sigma$. For $\sigma > 1$ the approximation degrades slowly. %To improve its accuracy further the power coefficient can be made variable from $2$ at $\sigma\rightarrow 0$ to $1$ at $\sigma \rightarrow \infty$.

\paragraph{Approximation for Large Variance}
For $\sigma \geq 2$, the following approximation is more accurate. In order to compute $E_X \S^2(X)$ we apply the same trick as with expectation of $\S(X)$. Observe that $\S^2(X)$ can be considered as a cdf and let $Z$ have this distribution. Then $E_X \S^2(X) = E_{X,Z}\leftbb X-Z \geq 0 \rightbb$. We calculate that $\mu^Z = 1$ and $\sigma^Z = (\pi^2/3 - 1)^\frac{1}{2}$. Assuming that $X-Z$ is distributed normally with mean $\mu^X-1$ and variance $(\sigma^X)^2 + \pi^2/3 - 1$ we get the approximation for the variance
\begin{align}
\Phi\big((\mu - 1) / (\sigma^2 + \sigma_\S^2-1)^{\frac{1}{2}} \big) - \mu'^2.
\end{align}
This is not directly usable in practice because of the error functions and needs to be approximated further.
\begin{figure*}[!t]
\includegraphics[width=0.2\linewidth]{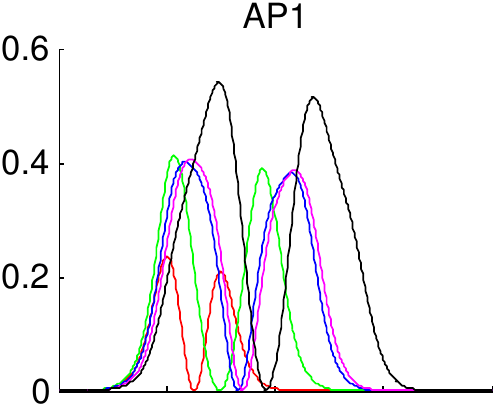}%
\includegraphics[width=0.2\linewidth]{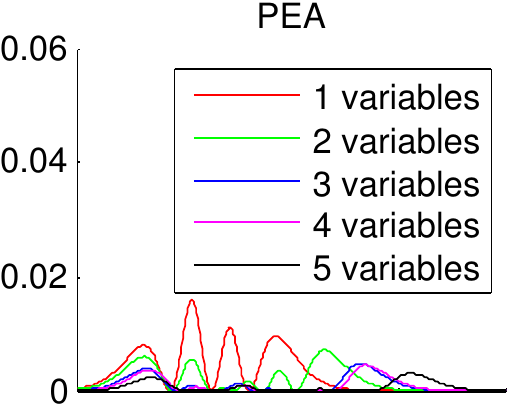}%
\includegraphics[width=0.2\linewidth]{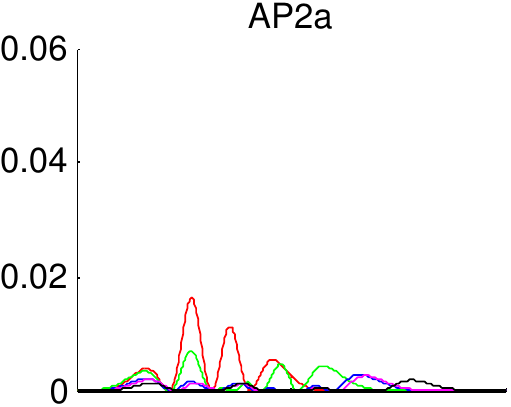}%
\includegraphics[width=0.2\linewidth]{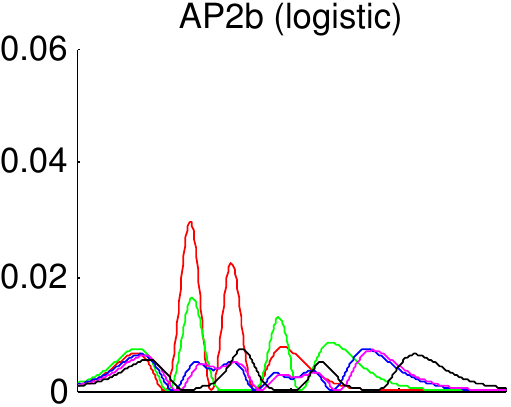}%
\includegraphics[width=0.2\linewidth]{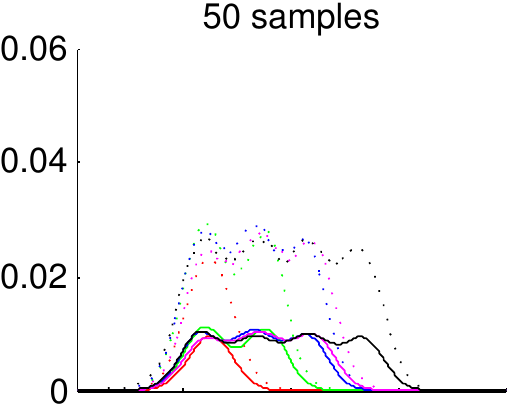}%
\caption{Comparison of approximations for logistic unit in~\cref{fig:density-example}. Each line shows KL divergence to the true posterior (in bits) as a function of the bias. Curves are plotted for 1-5 input variables $X_i$. Note the 10-fold scaling of y-axis of AP1 compared to others. Note that all approximations: PEA and AP2a, AP2b achieve about the same accuracy in this test, comparable to the estimate using 50 samples (solid lines show mean KL divergence over 1000 repeated trials, dashed lines show 90 percent confidence interval of the trials).
\label{fig:approx-eval}
}
\end{figure*}
\begin{figure}[t]
\centering
\includegraphics[width=0.49\linewidth]{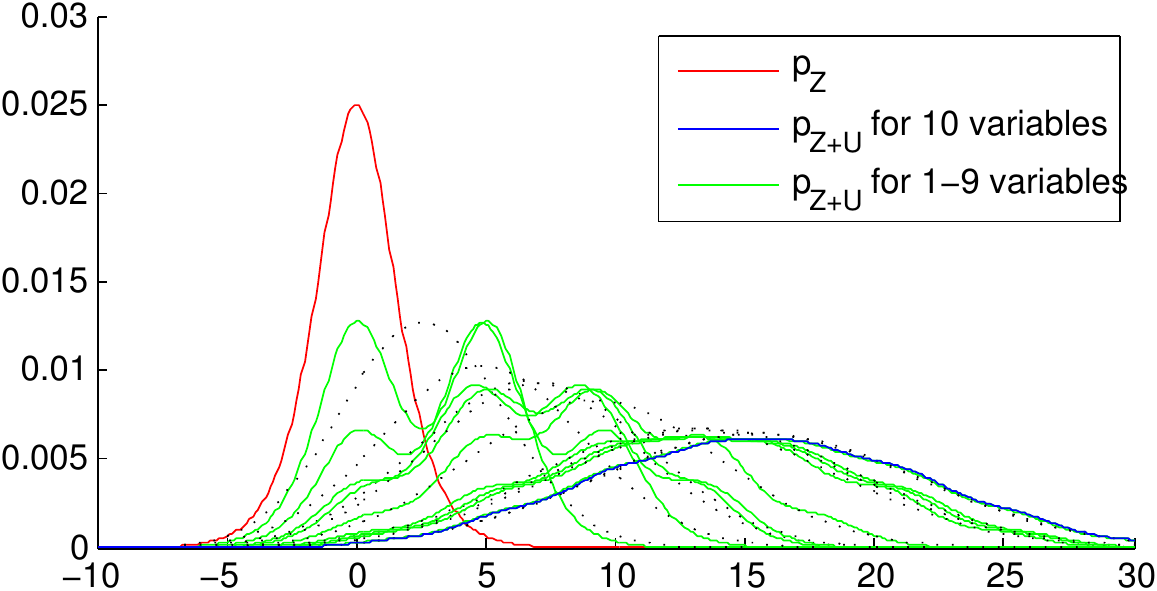}
\includegraphics[width=0.49\linewidth]{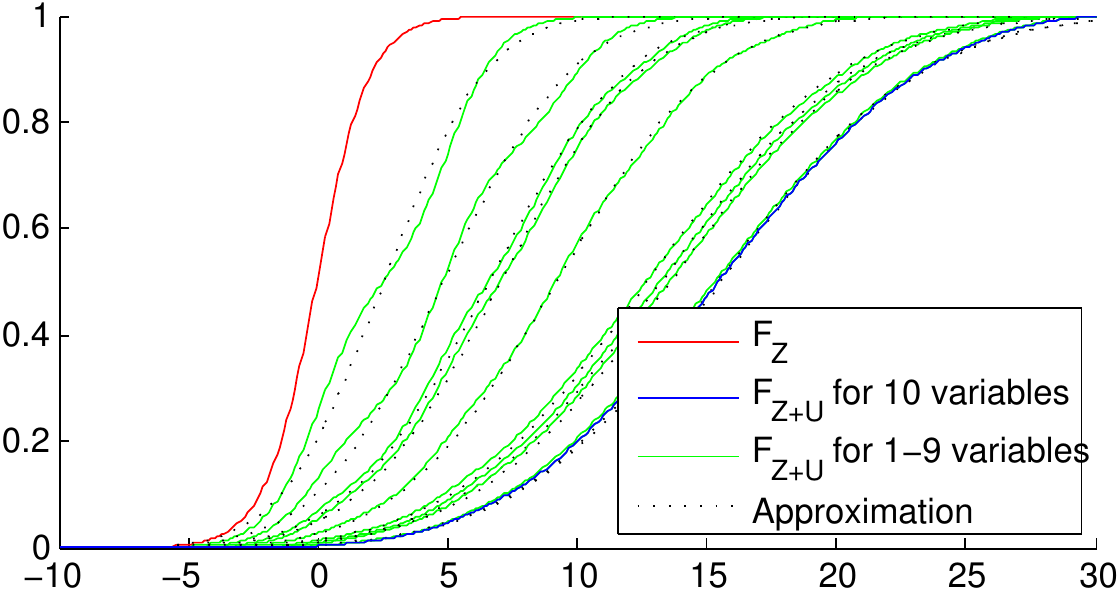}\\
$w = (5.1, 4.6, 3.5, 0.9, 4.3, 7.0, 1.1, 0.7, 3.6, 0.3)$
\caption{
Illustration to \cref{prop:gaussian-input}: an example of the distribution of the sum $U+Z$, where $U = \sum_{i=1}^{n}w_i X_i$ for 1 to $10$ independent Bernoulli variables $X_i$ with $p(X_i{=}1) = 0.5$ and $Z$ has a logistic distribution. Weights $w$ are displayed in the bottom (drawn once from $U[0,10]$).
{\em Left}: densities. With $2$ inputs the distribution is bimodal, but with more inputs it quickly tends to normality.
{\em Right}: cdfs of $Z$ and $Z+U$ and the approximating normal cdf~\cref{AP2a}. Note that due to the symmetry of the Logistic density, $U+Z$ and $U-Z$ have identical distributions.
\label{fig:density-example}}
\end{figure}
%
%==============================================================================
\subsection{ReLU}\label{detail:relu}
The mean of $\max(0,X)$ expresses as
\begin{align}
\mu' = \int_{-\infty}^{\infty} \max(0,x) p(x) \d x = \int_{0}^{\infty} x p(x) \d x.
\end{align}
%Normal approximation:
Consequently, assuming $X$ to be normally distributed, we get
\begin{align}
&\int_{0}^{\infty} x \phi((x-\mu)/\sigma ) \d x = \int_{-\mu /\sigma }^{\infty} (\mu + \sigma x) \phi(x) \d x\\
&=\mu \int_{-\mu/\sigma }^{\infty}  \phi(x) \d x - \sigma  \int_{-\mu/\sigma }^{\infty} \phi'(x) \d x \\
& =\mu \Phi(\mu/\sigma ) + \sigma  \phi(\mu/\sigma ).
\end{align}
where we used that $x \phi(x) = \phi'(x)$ holds for the pdf of the standard normal distribution,

Second moment:
\begin{align}
&\int_{0}^{\infty} x^2 \phi((x-\mu)/\sigma ) \d x  = \int_{-\mu/\sigma}^{\infty} (\sigma x+\mu)^2 \phi(x) \d x\\
&= \int_{-\mu/\sigma}^{\infty} (\sigma^2 x^2 + 2 \sigma \mu x +\mu^2) \phi(x) \d x\\
\notag
&= \sigma^2 \int_{-\mu/\sigma}^{\infty} -x \phi'(x) \d x + 2\sigma \mu \phi(\mu/\sigma) + \mu^2 \Phi(\mu/\sigma) \\
\notag
&= \sigma^2 \Big( \frac{\mu}{\sigma} \phi\left(\frac{-\mu}{\sigma}\right) + \Phi\left(\frac{\mu}{\sigma}\right) \Big)  + 2\sigma \mu \phi\left(\frac{\mu}{\sigma}\right) + \mu^2 \Phi\left(\frac{\mu}{\sigma}\right)\\
&= \sigma \mu \phi(\mu/\sigma) + (\mu^2 + \sigma^2) \Phi(\mu/\sigma).
\end{align}
Let us denote $a = \mu/\sigma$. Then $\mu' = \sigma (a\Phi(a)+\phi(a))$. The variance in turn expresses as
\begin{align}
\sigma'^2 = \sigma^2 \R(a),
\end{align}
where 
\begin{align}\label{R-function}
\R(a) = a \phi(a) + (a^2+1)\Phi(a) - (a\Phi(a)+\phi(a))^2.
\end{align}
This function is non-negative by the definition of variance. However it contains positive and negative terms and in a numerical approximation the errors may not cancel and result in a negative value. In practice this function of one variable can be well-approximated by a single logistic function, and guaranteed to be non-negative in this way.

Logistic approximations are computed with Mathematica. The two approximations for the mean are illustrated in~\cref{fig:relu}. The logistic approximation of the mean was also given in~\cite[Prop.1]{Bengio2013EstimatingOP}. The two approximations appear very similar. ELU activation~\cite{ClevertUH15} can be also seen as an approximation of the expectation, if we disregard the horizontal and vertical offsets, which can be influenced by the biases before and after the non-linearity.
\begin{figure}[t]
{\centering
\includegraphics[width=0.5\linewidth]{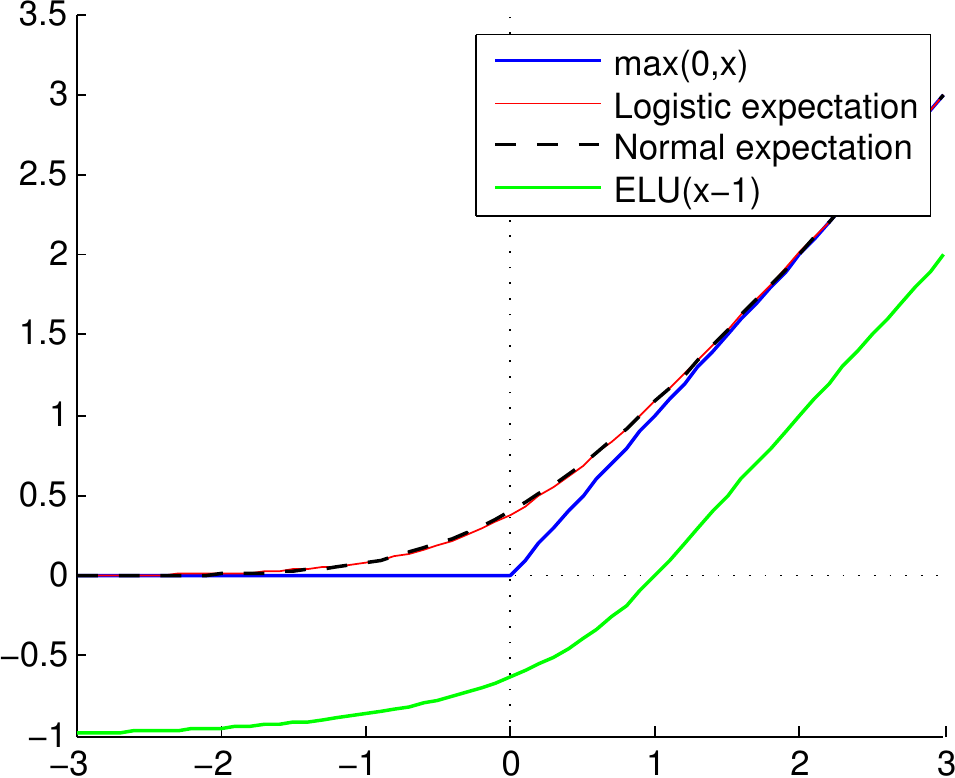}
}
\caption{The mean of $\max(0,X)$, under the approximation that $X$ is distributed Normally or logistically, in both cases with variance 1. The two approximations are very similar. For comparison also the ELU activation~\cite{ClevertUH15} is shown, with an input offset by 1.
\label{fig:relu}}
\end{figure}
%==============================================================================
\subsection{Max and Leaky ReLU}\label{sec:max}
We will consider the function $\max(X_1, X_2)$ in the two special cases: when $X_2 = \alpha X_1$, \ie, they are fully correlated, and when $X_1$ and $X_2$ are assumed independent. The first case is useful for representing leaky ReLU, given by $\max(X, \alpha X)$ and the second case may be used to handle cases where we don't know the correlation, \eg max pooling and {\tt maxOut}.
We use moments for the maximum of two correlated Gaussian random variables given in~\cite{NadarajahK08}. Denoting $s = (\sigma_1^2+\sigma_2^2 - 2 \Cov[X_1,X_2])^\frac{1}{2}$ and $a = (\mu_1{-}\mu_2)/s$, the mean and variance of $\max(X_1,X_2)$ can be expressed as:
\begin{subequations}
\begin{align}
\mu' & = \mu_1 \Phi(a) + \mu_2 \Phi(-a)  + s \phi(a),\\
\label{max-var-e1}
\sigma'^2 & 
\begin{aligned}[t]
=\:& (\sigma_1^2+\mu_1^2)\Phi(x) + (\sigma_2^2+\mu_2^2)\Phi(-a)\\
+\:&(\mu_1 + \mu_2)s\phi(a) - \mu'^2.
\end{aligned}
\end{align}
\end{subequations}
The mean can be expressed as:
\begin{align}\label{max-mean-e1}
\mu' = \mu_2 + s \big( a \Phi(a) + \phi(a) \big).
\end{align}
Substituting this expression into~\eqref{max-var-e1} we obtain
\begin{align}\label{max-var-gen}
\sigma'^2 = \sigma_1^2 \Phi(a) + \sigma_2^2\Phi(-a) + s^2 (a^2 \Phi(a) + a \phi(a) - (a\Phi(a) +\phi(a))^2 )
\end{align}
Reusing the function $\R$~\eqref{R-function}, the variance expresses as:
\begin{align}\label{max-var-gen2}
\sigma'^2 = \sigma_1^2 \Phi(a)+ \sigma_2^2 \Phi(-a) + s^2 (\R(a) - \Phi(a)).
\end{align}
\paragraph{Leaky ReLU}
We now can derive a simplified expression for {\tt LReLU}. Assume that $\alpha <1$, let $X_2 = \alpha X_1$ and denote $\mu = \mu_1$ and $\sigma^2 = \sigma_1^2$. Then $\mu_2 = \alpha \mu$, $\sigma_2^2 = \alpha^2 \sigma^2$ and $\Cov[X_1,X_2] = \Cov[X_1,\alpha X_1] = \alpha \sigma^2$. 
We then have $s = \sigma (1-\alpha)$ and $a = (\mu_1-\mu_2)/s = \mu (1-\alpha)/s = \mu/\sigma$.
The mean $\mu'$ expresses as
\begin{align}
\mu' = \mu (\alpha + (1-\alpha) \Phi(a)) + \sigma(1-\alpha) \phi(a).
\end{align}
The variance $\sigma'^2$ expresses as
\begin{align}
& \sigma^2\Big( \Phi(a) + \alpha^2 (1-\Phi(a)) + (1-\alpha)^2 \big(a^2\Phi(a)  
+ a\phi(a) - (a \Phi(a) +\phi(a))^2\big) \Big)\\
& =\sigma^2(\alpha^2  + 2\alpha(1-\alpha)\Phi(a)+ (1-\alpha)^2 \R(a)).
\end{align}
In practice we approximate it with the function 
\begin{align}\label{LReLU-approx}
\sigma'^2 \approx \sigma^2( \alpha^2 + (1-\alpha^2)\S(a / t) ),
\end{align}
where $t$ is set by fitting the approximation (see~\cref{sec:implementation}). The approximation is shown in~\cref{fig:LReLU}.
%\begin{figure}
%\includegraphics[width=\linewidth]{fig/prop/lrelu-0-0}\\
%\includegraphics[width=\linewidth]{fig/prop/lrelu-0-1}
%\caption{\label{fig:LReLU} 
%Mean and variance of leaky ReLU: $Y = \max(X, \alpha X)$ with $\alpha=0$ (top) and $\alpha=0.1$ (bottom) using approximation~\eqref{LReLU-approx}.}
%\end{figure}
\paragraph{Uncorrelated Case} In this case we have $s^2 = \sigma_1^2+ \sigma_2^2$. The expression for the variance~\eqref{max-var-gen} can be written as 
\begin{align}
%&\sigma'^2 = \sigma_1^2 \Big( \Phi(a) + a^2 \Phi(a) + a \phi(a) - (a\Phi(a) -\Phi(a))^2 \Big) \\
%            +& \sigma_2^2 \Big( \Phi(-a) + a^2 \Phi(a) + a \phi(a) - (a\Phi(a) -\Phi(a))^2 \Big).
\sigma'^2 = \sigma_1^2 \R(a)+ \sigma_2^2 (\Phi(-a) + \R(a) - \Phi(a)).
\end{align}
%Both functions $\R(a)$ and $\bar\R(a) = \Phi(-a) + \R(a) - \Phi(a)$.
This expression can be well approximated with
\begin{align}
\sigma_1^2 \S(a/t) + \sigma_2^2 \S(-a/t)
\end{align}
with a suitable parameter $t$. For the purpose of visualization, consider applying this approximation for computing the moments of $\max(X, \alpha X)$, ignoring the correlation. It will result in a plot similar to~\cref{fig:LReLU} but with a slightly increased variance in the transition part and with a slightly more smoothed mean.

%This form is numerically more suitable, since all summands can be forced to be non-negative uni-variate functions in a numerical approximation.

%Furthermore, for $\sigma^2\rightarrow 0$ the 
%
%The exponential term is set empirically. In the limit of small $\sigma$ the true dependence on $\mu'$ should be the squared derivative of the logistic function -- the squared logistic density -- and for large $\sigma$ the true dependence should approach $\mu'(1-\mu')$ same as in Heaviside step function. The approximation behaves reasonably in both limits but there are no uniform bounds.
%
%The best empirical results we got by fitting the function
%\begin{align}
%\mu'(1-\mu') t_0 + \exp( -\mu^2/(\sigma^2+\sigma_\S) t_1)t_2
%\end{align}
%to the Monte Carlo integrated data points.
%
%==============================================================================
\subsection{Softmax}\label{detail:softmax}
For the posterior of softmax\footnote{The established term {\em softmax} is somewhat misleading, since the hard version of the function computes not $\max$ of its arguments but $\argmax$ in a form of indicator.} $p(Y{=}y\mid X) = \exp(X_y)/ \sum_{i} \exp(X_i)$ we need to estimate %$p(Y{=}y \mid x^0) = $
\begin{align}\label{softmax-prob}
\E_{X} \Big[ \frac{ e^{X_y} }{\sum_k e^{X_k} } \Big] = \E_{X} \Big[ \frac{ 1 }{1 + \sum_{k\neq y} e^{X_k - X_y} } \Big].
\end{align}
Let $U_k = X_y - X_k$ for $k \neq y$, so that $U$ is a \rv in $\Real^{n-1}$. Let us assume for simplicity that $y = n$.
Expression~\eqref{softmax-prob} can be written as 
\begin{align}\label{softmax-mnv-l}
\E_{U} \Big[ \big(\textstyle 1 + \sum_{k < n} e^{-U_k} \big)^{-1} \Big] = \E_{U} [\S_{n-1}(U)],
\end{align}
where $\S_{n-1}$ is the cdf of the $(n{-}1)$-variate logistic distribution~\cite[eq. 2.5]{Malik-73}:
\begin{align}\label{mnv-l}
\S_{n-1}(u) = \frac{1}{1+ \sum_k e^{-u_k}}.
\end{align}
We can apply the same trick as in~\cref{O:logistic-threshold}. 
Let $Z \sim \S_{n-1}$. Then 
\begin{align}\label{softmax-mnv-l-q}
\E_{U} \S_{n-1}(U) = \E_{U,Z} \leftbb U-Z \geq 0 \rightbb.
\end{align}

%Note that the components of $U$ are not independent even assuming that the components of $X$ are, and in addition the components of 
%We can compute the mean and variance of $U-Z$ from those of $U$, $Z$ since we know $\E[U_k] = \mu_y-\mu_k$ and $\Var[U_k] = \sigma^2_y+\sigma^2_k$ and can compute $Z$.

%If the components of $X$ are independent, then so are the components of $U$ and $\E[U_k] = \mu_y-\mu_k$ and $\Var[U_k] = \sigma^2_y+\sigma^2_k$. The components of $Z$ are not independent, but we make a simplifying assumption that they are. 
%We make a simplifying assumption that the components of $U-Z$ are independent (in general neither components of $U$ are independent for independent $X$ nor the components of $Z$). 
%
For multi-variate logistic distribution of $Z$, the marginal distribution of $Z_k$ is logistic~\cite{Malik-73}, hence we know $Z_k$ has mean $0$ and variance $\sigma_S^2$. We can thus express the mean of $(U-Z)_k$ as $\tilde \mu_k = \mu_y-\mu_k$ and its variance as $\tilde \sigma^2_k = \sigma^2_y + \sigma^2_k + \sigma_S^2$ (this relies only on that $X_y$, $X_k$ and $Z_k$ are independent).
Note in, general, the components of $U -Z$ are not independent.
%Then the components of $U-Z$ are independent with mean $\tilde \mu_k = \mu_y-\mu_k$ and variance $\tilde \sigma^2_k = \sigma^2_y + \sigma^2_k + \sigma_S^2$.

The normal approximation, similar to~\cref{AP2a} is as follows. Assuming that $U-Z$ has multivariate normal distribution with diagonal covariance (which implies that the components of $U -Z$ are independent), gives the approximation:
\begin{align}
q(y) = \Pr\{U{-}Z \geq 0\} \approx 
\prod_{k \neq y} \Phi \Big(\frac{\tilde \mu_k}{\tilde \sigma_k}\Big).
\end{align}
Expanding, we obtain
\begin{align}
q(y) \approx \prod_{k \neq y} \Phi \Big(\frac{\mu_y-\mu_k}{\sqrt{\sigma^2_y + \sigma^2_k + \sigma_S^2}}\Big).
\end{align}

The logistic approximation, similar to~\cref{AP2b} is as follows. 
%In order to estimate the expectation of the indicator in~\eqref{softmax-mnv-l-q} 
Assuming that $U-Z$ has multivariate logistic distribution with the matching mean and the diagonal elements of the covariance matrix, we can approximate
%\begin{align}\label{softmax-mnv-l-q}
%\E_{U} \S_{n-1}(U) \approx \S_{n-1}((\Sigma^Z)^\frac{1}{2} \tilde \Sigma^{-\frac{1}{2}} \tilde \mu),
%\end{align}
%where $\tilde \mu$ = $\mu^U + \mu^Z = \mu^U$ and $\tilde \Sigma = \Sigma^U + \Sigma^Z$. We make a simplifying assumption that both $\Sigma^U$ and $\Sigma^Z$ are diagonal. This is well in order with ignoring correlations of variables in all layers. The resulting simplified approximation is
\begin{align}\label{softmax-mnv-l-q1}
q(y) \approx \S_{n-1}\big( \{ \tilde\mu_k / (\tilde\sigma_k/\sigma_\S) \}_k \big).
\end{align}
%Expanding $\mu^U$, $\Sigma^U$, 
Expanding, we obtain the approximation
%get for the stochastic softmax output $q(y)=$
%\begin{align}\label{softmax-ap1}
%%\notag
%\Big(  1+ \sum_{k\neq y}\exp\Big\{ \frac{\mu_k - \mu_y}{ \sqrt{ \sum_{i}\Var[X_k - X_y]/\Sigma_\S+1}  } \Big\}  \Big)^{-1}.
%\end{align}
%If we let $\Var[X_k - X_y]$ to be approximated as $\Var[X_k] + \Var[X_y]$, assuming $X_k$ and $X_y$ are independent, we obtain the approximation
\begin{align}\label{softmax-ap2}
%\notag
q(y) \approx \Big(  1+ \sum_{k\neq y}\exp\Big\{ \frac{\mu_k - \mu_y}{ \textstyle\sqrt{ (\sigma^2_k + \sigma^2_y)/\sigma^2_\S+1}  } \Big\}  \Big)^{-1}.
\end{align}
In both cases, a renormalization of $q$ is needed in order to ensure a proper distribution. This is not guaranteed by the approximation as was the case with~\eqref{AP2a}, \eqref{AP2b}.
%Note that in the case of two classes the estimate~\eqref{softmax-ap1} recovers back the approximation~\cref{AP2b} but~\eqref{softmax-ap2} does not.
%A better estimate of $\Var [X_k - X_y]$ may be obtained by recalling that $X$ is formed by a linear transform: $X = W \hat X$ and expanding $\Var [X_k - X_y]$ as $\sum_{i}(w_{ki} - w_{y i})^2\hat\sigma_i^2$, where $\hat\sigma_i^2$ is variance of $\hat X_i$.
%
%\par
%Further simplification may be obtained by approximating $\sum_{i}(w_{ki}-w_{y i})^2\sigma_i^2 \approx \sum_{i}(w_{ki}^2 + w_{y i}^2)\sigma_i^2$. In this case we can split the linear transform $\hat X = W X$, with $\hat \mu = W X$ and $\hat \sigma_j^2 = \sum_{i} w_{ji}^2 \sigma_i^2$. This results in the approximation
%
Approximation~\eqref{softmax-ap2} can be implemented in the logarithmic domain as follows:
\begin{align}
\log q(y) := - \logsumexp_k \frac{\mu_k-\mu_y}{\sqrt{(\sigma_k^2+\sigma_y^2)/\sigma_S^2+1}},
\end{align}
where $\logsumexp_k$ is $\log \sum_k \exp$ operation. This can be done in quadratic time and with linear memory complexity. The renormalization of $q$ in the logarithmic domain takes the form
\begin{align}
\log q(y) := \log q(y) - \logsumexp_k \log q(k).
\end{align}
The log likelihood~\eqref{log-likelihood} can take $\log q(y)$ directly, avoiding the exponentiation, as with regular softmax.
It turned out that back propagation through this softmax was rather slow and we replaced it with a simplified approximation
\begin{align}
\log q(y) := -\logsumexp_k \Big(\frac{\mu_k}{\sqrt{\sigma_k^2/\sigma_S^2+1}} - \frac{\mu_y}{\sqrt{\smash[b]{\sigma_y^2}/\sigma_S^2+1}} \Big),
\end{align}	
which reduces to standard softmax of $\mu_k/\sqrt{\sigma_k^2/\sigma_S^2+1}$. There is a noticeable loss of accuracy in the KL divergence of the posterior distribution seen in the experiments (\cref{tab:accuraacy-LeNet,tab:accuracy-CIFAR}), which is however not critical. We therefore used this expression in the end-to-end training experiments.

\paragraph{Connection to argmax}
We now explain a refinement of the latent variable model for softmax, which allows to see its interpretation as the expected value of $\argmax$.
%The components of $Z$ are not independent, 
In fact, this interpretation is well known in particular in the multinational logistic regression, we just expose the relation between these models.		

As mentioned above, the components of $Z$ following the $(n-1)$-variate logistic distribution are not independent, but they have the following latent variable representation~\cite[eq. 2.1-2.4]{Malik-73}.
Let $\alpha$ has the exponential density function $e^{-\alpha}$, $\alpha >0$. Let variables $Z_i$ given $\alpha$ be independent with the conditional distribution
\begin{align}\label{Malik-conditional-cdf}
F_{Z_k \mid \alpha}(z) = e^{-\alpha e^{-z}}.
\end{align}
Then the joint marginal distribution of $Z$ is the multivariate logistic distribution $\S_{n-1}$. This latent variable model can in turn be rewritten in the following form, very suitable for our purpose.
\begin{lemma}\label{lemma:mnv-logistic}
Let $V_i$ be independent \rv's with the exponential density $p_{V_i}(u) = e^{-u}$, $u > 0$.
Let $Z_{i} = \log(V_n) - \log(V_i)$ for i = $1,\dots, n-1$. Then $Z$ has $(n{-}1)$-variate logistic distribution.
\end{lemma}
\begin{proof}
%Let us denote $\alpha = V_n$ for consistency with~\cite{Malik-73}. 
It can be verified that the cdf of $-\log(V_i)$ expresses as $F_{-\log(V_i)}(u) = $
\begin{align}\label{ev1-cdf}
 Pr\{ -\log(V_i) \leq u\} = Pr\{ V_i \geq e^{-u}\} = \int_{e^{-u}}^{\infty} e^{-x} dx = e^{-e^{-u}}.
\end{align}
The conditional cdf of $Z_i$ given $V_n = \alpha$, in turn expresses as 
\begin{align}
 F_{Z_i}(z) = Pr\{\log(\alpha) - \log(V_i) \leq z\} =  F_{-\log(V_i)}(z - \log(\alpha))\\
= e^{-e^{-z + \log(\alpha)}} = e^{-\alpha e^{-z}},
\end{align}
which matches~\eqref{Malik-conditional-cdf}.
\end{proof}
%Consequently, the softmax $p(Y{=}y\mid X) = \exp(X_y)/\sum_{k\neq y} \exp(X_k)$ has a latent variable formulation
%Consequently, the softmax has a latent variable formulation
In combination with~\eqref{softmax-mnv-l}, the softmax has a latent variable formulation $p(Y{=}y \mid X = x) = $
\begin{align}\label{softmax-lvm-argmax}
 & \E \big[ x_y - x_k - Z_k \geq 0 \ \forall k \big] = \E \big[ x_y - \log V_y - (x_k - \log V_k)  \geq 0  \ \forall k \big]\\
\notag
= & \E \big[ x_y + \Gamma_y \geq x_k + \Gamma_k  \ \forall k \big],
\end{align}
where $\Gamma_k = -\log V_k$ 	for $k=1,\dots n$ are independent noises with cdf~\eqref{ev1-cdf} known as Gumbel or type I extreme value distribution. It follows that the softmax model $p(Y \mid X)$ can be equivalently defined as
\begin{align}\label{softmax-lvm-argmax}
Y = \argmax_k (X_k + \Gamma_k).
\end{align}
Denoting $\tilde X_k = X_k + \Gamma_k$ the (additionally) noised input variables, we can express
\begin{align}
p(Y {=} y) = \E \big[p(Y {=} y \mid X) \big] = \E \Big[ \argmax_k \tilde X_k = y \Big].
\end{align}
%In combination with~\eqref{softmax-mnv-l}, this gives the following representation. 
%\begin{lemma}
%Let $X_k$, $k=1\dots n$ be the input variables to the softmax and $V_k$ are independent exponential \rv's. %Let $U_k = X_y - X_k$ for $k \neq y$. Then
%Let $\tilde X_k = X_k - \log(V_k)$. Then the expectation of softmax output for class $y$ is given by
%\begin{align}\label{softmax-argmax}
%\E_{X} \Big[ \frac{ e^{X_y} }{\sum_k e^{X_k} } \Big] = \E_{\tilde X} \leftbb \tilde X_y \geq \tilde X_k \ \ \forall \ \ k\neq y \rightbb.
%\end{align}
%\end{lemma}
%\begin{proof}
%Using representation~\eqref{softmax-mnv-l} and~\cref{lemma:mnv-logistic}.
%\end{proof}
%Notice the the condition in the right hand side $\tilde X_y \geq \tilde X_k$ expresses precisely the fact that $\tilde X_y$ is the maximum of all $\tilde X_k$. 
Therefore the problem of computing the expectation of softmax has been reduced to computing the expectation of $\arg \max$ for additionally noised inputs. This connection is very similar to how expectation of sigmoid function was reduced to the expectation of the Heaviside step function with additional injected noise input.
%

%Note, that by the construction in~\cref{lemma:mnv-logistic} the mean parameters of $\log(V_i)$ cancel. Within our approximation, it is sufficient to know the variance of $\log(V_i)$ for exponential $V_i$, which equals $\pi^2/6$.

%
%
%We did not observe a significant difference in the results with this simplified approximation. All results reported in the paper are using it.
%==============================================================================
\subsection{Probit}\label{sec:probit}
Let us consider {\em probit} model: $Y$ is a Bernoulli \rv with $p(Y{=}1\mid X) = \Phi(X)$.
It has the latent variable interpretation $Y = X - Z$, $Z \sim \Phi$. In our approximation~\cref{AP2a} it changes only the value of noise variance, which gives a simplified formula 
\begin{align}
\E [\Phi(X)]		 \approx \Phi\Big( \frac{\mu}{\sqrt{\sigma^2 + 1}} \Big).
\end{align}
Note that the approximation that $X-Z$ is normally distributed becomes more plausible.
%
%===============================================================================
\section{Experiment Details}\label{sec:extraexp}
%===============================================================================
In this section we give all details necessary to ensure reproducibility of results and auxiliary plots giving more details on the experiments in the main part.
%===============================================================================
\subsection{Implementation Details}\label{sec:implementation}
We implemented our inference and learning in the pytorch\footnote{\url{http://pytorch.org}} framework. The source code will be published together with the paper.
The implementation contains a number of layers {\tt Linear, Conv2D, Sigmoid, ReLU, SoftMax, Normalize}, \etc, that input and output a pair of mean and variance and can be easily used to upgrade a standard model made of such layers. At present we use only higher-level pytorch functions to implement these layers. For example, convolutional layer is implemented simply as
\lstset{basicstyle=\footnotesize}
\begin{lstlisting}[language=Python]
	y.mean = F.conv2d(x.mean, w) + b
	y.var = F.conv2d(x.var, w*w)
\end{lstlisting}

The ReLU variance function $\R(x)$~\eqref{R-function}, which is also used in leaky ReLU, was approximated by a single logistic function
\begin{lstlisting}[language=Python]
	F.sigmoid(x /0.3729) 
\end{lstlisting}
%(x - 0.3766) * 1.8482)
fitted to minimize the maximum KL divergence for LReLU(0.01).
The cdf of the normal distribution was approximated by the cdf of the logistic distribution as $\Phi(x) \approx \S(x \pi/\sqrt{3})$, \ie, by matching the variance. Under this approximation the means of logistic-Bernoulli and logistic transform are the same for bith~\eqref{AP2a} and~\eqref{AP2b}. For the variance of logistic transform we used the expression~\eqref{logistic-heuristic-approx}.
\par
%Stochastic softmax~\eqref{table:approx} can be implemented in the logarithmic domain as follows:
%\begin{align}
%\log q(y) := - \logsumexp_k \frac{\mu_k-\mu_y}{\sqrt{(\sigma_k^2+\sigma_y^2)/\sigma_S^2+1}},
%\end{align}
%where $\logsumexp_k$ is $\log \sum_k \exp$ operation. This can be done in quadratic time and linear memory complexity. The renormalization of $q$ in logarithmic domain takes the form
%\begin{align}
%\log q(y) := \log q(y) - \logsumexp_k \log q(k).
%\end{align}
%The log likelihood~\eqref{log-likelihood} then can use $\log q(y)$ without exponentiation, as usual. 
For numerical stability, it was essential that $\logsumexp$ is implemented by subtracting the maximum value before exponentiation
\begin{lstlisting}[language=Python]
 m, _ = x.max()
 m = m.detach() # does not influence gradient
 y = m + torch.log(torch.sum(torch.exp(x - m)))
\end{lstlisting}
\par
The relative times of a forward-backward computation in our higher-level implementation are as follows:
\begin{lstlisting}
	standard training      1
	BN                     1.5
	inference=AP2          3
	inference=AP2-norm=AP2 6
\end{lstlisting}
Please note that these times hold for unoptimized implementations. In particular, the computational cost of the AP2 normalization, which propagates the statistics of a single pixel statistics, should be negligible in comparison to propagating a batch of input images.
%With a more efficient implementation we expect bottlenecks in memory accesses and convolutions and a performance about twice as slow as standard NN or better and AP2 normalization should incur only a negligible overhead.
%===============================================================================
\subsection{Parameters}
We used batch size 128, Adam optimizer with learning rate $10^{-3}0.06^{k}$, where $k$ is the epoch number. This schedule smoothly decreases the learning rate by about order of 10 every 50 epochs.
Parameters of linear and convolutional layers were initialized using pytorch defaults, \ie, uniformly distributed in $[-1/\sqrt{c},\, 1/\sqrt{c}]$, where c is the number of inputs per one output.
All experiments use a fixed random seed so that the initial point is the same. The results were fairly repeatable with arbitrary seeds.
%The input variance in varNN models was set to 0.1.
%===============================================================================
\subsection{Datasets}
We used MNIST\footnote{\url{http://yann.lecun.com/exdb/mnist/}} and CIFAR10\footnote{\url{https://www.cs.toronto.edu/~kriz/cifar.html}} datasets. Both datasets provide a split into training and test sets.
From the training set we split 10 percent (at random) to create a validation set. The validation set is meant for model selection and monitoring the validation loss and accuracy during learning. The test sets were currently used only in the stability tests.
%however we did not was used for model selection (we choose the model that achieves the best validation loss during learning) and to report the validation accuracy and loss at each epoch.

%=============================================================================== 
\subsection{Network specifications}
The MNIST single hidden layer network in~\cref{sec:experiments} MLP/MNIST has the architecture: input - FC 784x100 - Norm - Logistic Bernoulli - FC 100 x10 - Norm - Softmax, where {\tt Norm} may be {\tt none}, {\tt BN}, {\tt AP2}. With the normalization switched on, the biases of linear layers preceding normalizations are turned off.
\par
The LeNet in~\cref{sec:experiments} has the structure:
\begin{lstlisting}
Conv2d(1, 32, ks=5, st=2), Norm, Activation
Conv2d(32, 64, ks=5, st=2), Norm, Activation
Conv2d(64, 50, ks=4), Norm, Activation
Conv2d(50, 10, ks=1), Norm, LogSoftmax
\end{lstlisting}
Convolutional layer parameters list input channels, output channels, kernel size and stride. Dropout layers are inserted after activations.
\par
The CIFAR network in~\cref{sec:experiments} has a structure similar to LeNet with the following conv layers:
\begin{lstlisting}
ksize = [3,  3,  3,  3,   3,   3,   3,   1,   1 ]
stride= [1,  1,  2,  1,   1,   2,   1,   1,   1 ]
depth = [96, 96, 96, 192, 192, 192, 192, 192, 10]
\end{lstlisting}
each but the last one ending with Norm and activation. The final layers of the network are
\begin{lstlisting}
Norm, AdaptiveAvgPool2d, LogSoftmax
\end{lstlisting}
%===============================================================================
\subsection{Additional Experimental Results}
\paragraph{Stability to Adversarial Attacks}
\begin{figure}
\centering
\includegraphics[width=0.5\linewidth]{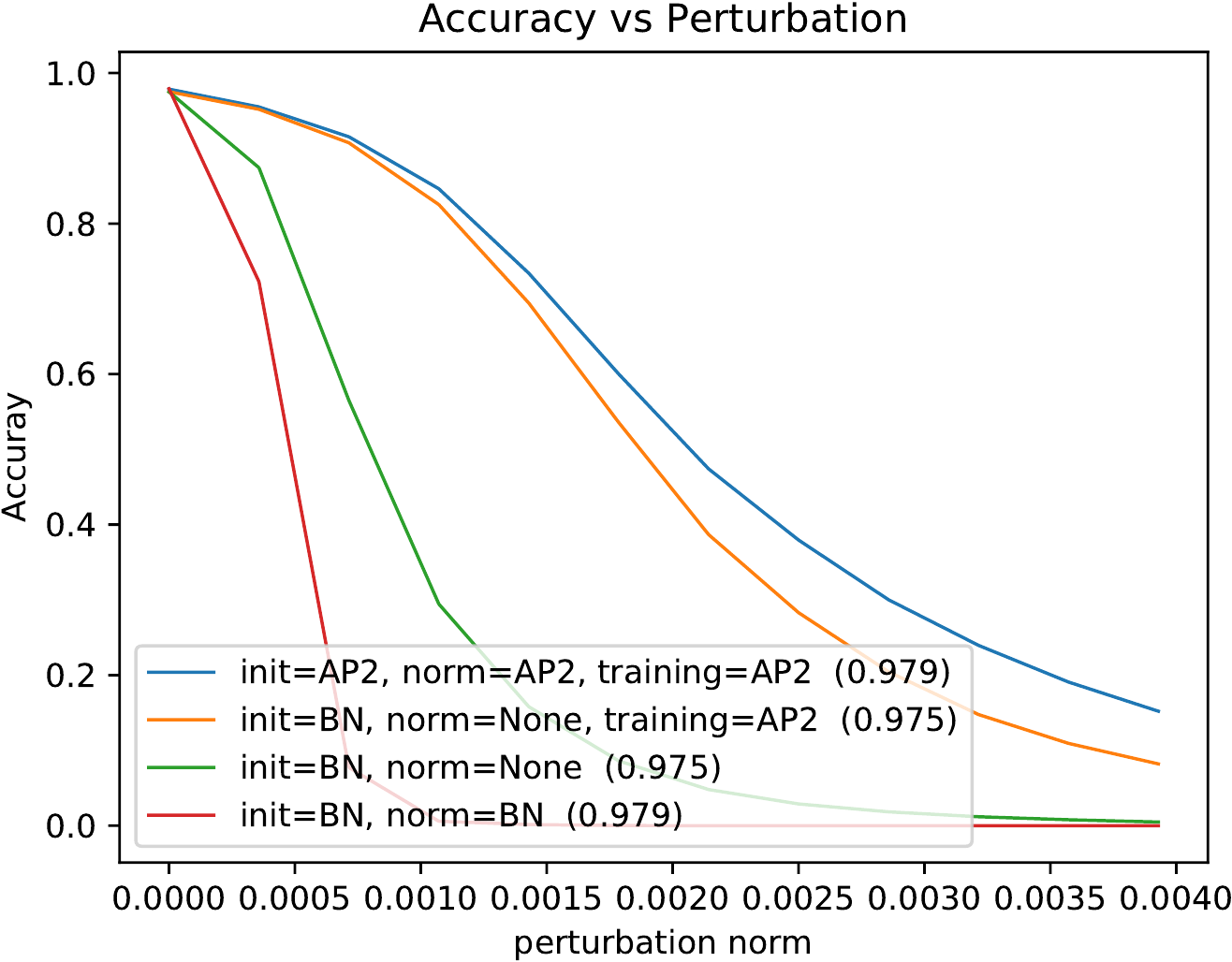}%
\includegraphics[width=0.5\linewidth]{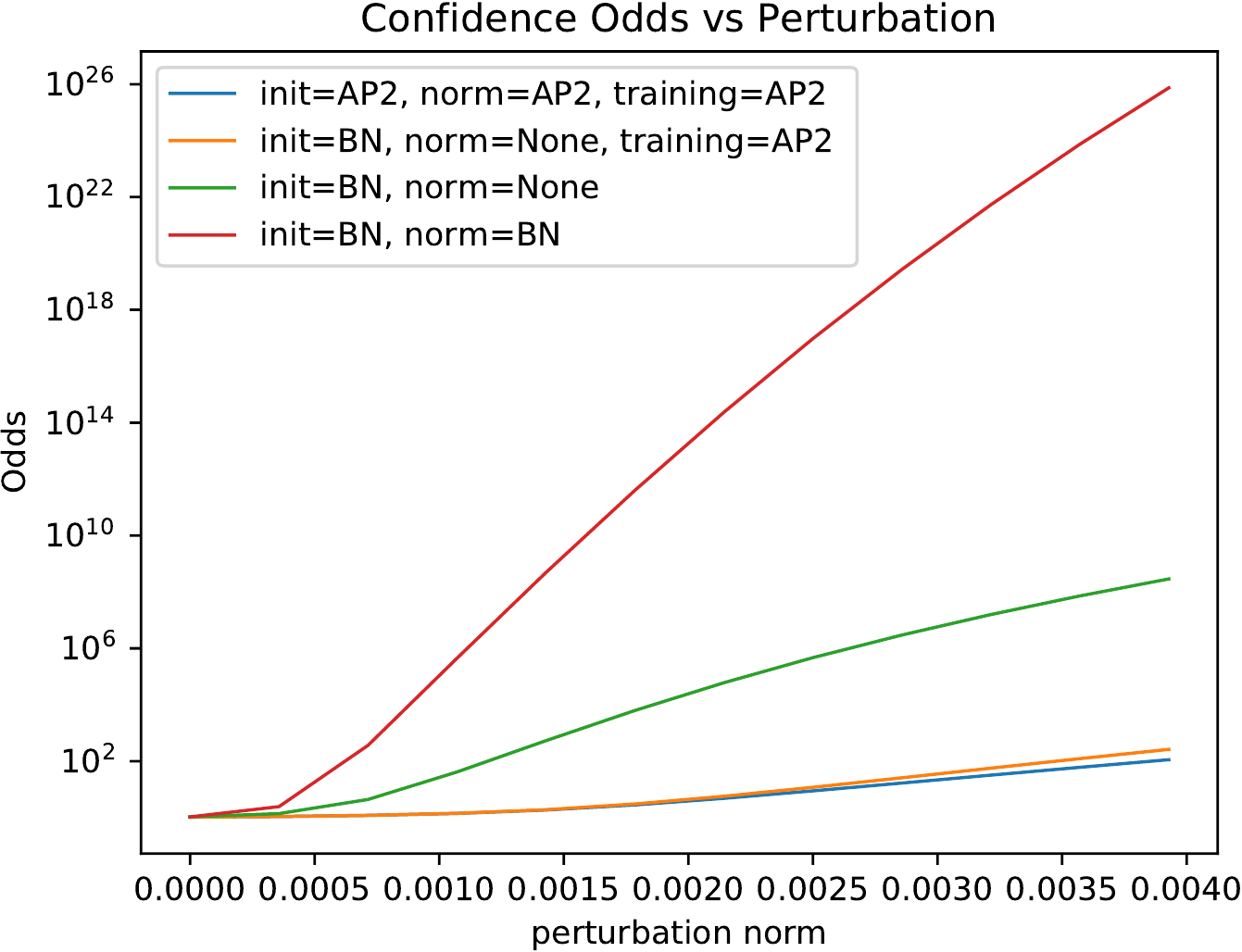}%
\\
\caption{\label{fig:stability-MLP-adv}
Stability evaluation of MLP/MNIST model under gradient sign attack. {\em Left}: test accuracy versus the maximum allowed norm of the adversarial perturbation who's direction is tailored per image. {\em Right} confidence odds as in~\cref{fig:stability-MLP}.
}
\end{figure}

%% file: tex/table_approx.tex
%
%\begin{table}
%\setlength{\tabcolsep}{0pt}
%\begin{tabular}{p{0.5\linewidth}p{0.5\linewidth}}
\setlength{\tabcolsep}{4pt}
\tabulinesep=3pt
%			\skipmyskip
%\newline\noindent
%\begin{figure*}
%{
%\scriptsize
%\setlength{\tabcolsep}{3pt}
%\setlength{\minrowclearance}{1pt}

%\tablehead{Header of first column & Header of second column \\}
%\begin{table}[!h]
%\caption{\label{table:mean-var}Propagating uncertainty in different layers.\vspace{-4ex}}
%\end{table}
\skipmyskip
\noindent
\begin{tabu}{|p{0.17\linewidth}|p{0.83\linewidth-4\tabcolsep}|}
\hline
\multicolumn{2}{|l|}{{\bf Linear}: $Y = W X$}\\
\hline
\multicolumn{2}{|p{\linewidth-2\tabcolsep}|}{
mean $\mu' = W \mu$;\ \  variance $\sigma'^2_j = \sum_i w_{ji}^2 \sigma_i^2$
}\\
%\hline
%mean & $\mu' = W \mu$ \\
%variance & $\Sigma' = \diag(W \Sigma W\T)$, 
%\newline in components
%variance & $\sigma'^2_j = \sum_i w_{ji}^2 \sigma_i^2$
% \\
\hline
\end{tabu}\\[-1.4pt]%%%%%%%%%%%%%%%%%%%%%%%%%%%%%%%%%%
%\skipmyskip
\begin{tabu}{|p{0.17\linewidth}|p{0.83\linewidth-4\tabcolsep}|}
\hline
\multicolumn{2}{|l|}{{\bf Heaviside} : $Y = \leftbb X \geq 0 \rightbb $; \cite{Frey-99}, \cref{detail:Heaviside}
}\\
\hline
mean $\mu'$ & Normal approx: $\Phi(\mu / \sigma )$,
\newline
Logistic approx: $\S(\mu / s)$, $s  = \sigma/\sigma_\S$
\\
\hline
variance & $\mu'(1 - \mu')$ \\
\hline
\end{tabu}\\[-1.4pt]%%%%%%%%%%%%%%%%%%%%%\\[0.5\baselineskip]
\begin{tabu}{|p{0.15\linewidth}|>{\raggedright}p{0.85\linewidth-4\tabcolsep}|}
\hline
\multicolumn{2}{|l|}{{\bf ReLU} : $Y = \max(0,X)$; \cite{Frey-99}, \cref{detail:relu} }\\
\hline
\multicolumn{2}{|c|}{Normal approx}\\
\hline
mean $\mu'$ &
$\mu \Phi(a)+\sigma \phi(a)$, where $a = \mu/\sigma$\\
\hline
variance &
$\sigma^2 \R(a) \approx \sigma^2 \S(a/t)$,
\newline
where $\R(a) = a\phi(a)\,{+}\,(a^2+1)\Phi(a)\,{-}\,(a\Phi(a){+}\phi(a))^2$ and $t$ is a fitted constant.\\
\hline
\multicolumn{2}{|c|}{Logistic approx}\\
\hline
mean $\mu'$ &
$s \log(1+\exp(\mu/s))$, where $s = \sigma/\sigma_\S$\\
\hline
variance & 
$-2 s^2 \Li_2(-e^{\mu/s}) - \mu'^2$, where $\Li_2$ is dilogarithm
%\newline
%$\mu^2 +\frac{\pi^2 s^2}{3} + 2 s^2 \Li_2(-e^{-\mu/s})
%where $\Li$ is the polylogarithm function.
\\
\hline
\end{tabu}\\[-1.4pt]%%%%%%%%%%%%%%%%%%%%%%%%%%%%%%%%%%%%%
\begin{tabu}{|p{0.15\linewidth}|>{\raggedright}p{0.85\linewidth-4\tabcolsep}|}
\hline
\multicolumn{2}{|l|}{{\bf LReLU} : $Y = \max(X,\alpha X)$; \cref{sec:max} }\\
\hline
\multicolumn{2}{|c|}{Normal approx}\\
\hline
mean $\mu'$ &
$\mu (\alpha + (1-\alpha) \Phi(a)) + \sigma(1-\alpha) \phi(a)$\\
\hline
variance &
$\sigma^2(\alpha^2  + (1-\alpha^2) \R(a))\approx \sigma^2( \alpha^2 + (1-\alpha^2)\S(a / t) )$\\
\hline
\end{tabu}\\[-1.4pt]%%%%%%%%%%%%%%%%%%%%%%%%%%%%%%%%%%%%%
%\\[0.5\baselineskip]
%\newline
%\skipmyskip
\begin{tabu}{|p{0.17\linewidth}|p{0.83\linewidth-4\tabcolsep}|}
\hline
\multicolumn{2}{|p{\linewidth-2\tabcolsep}|}{{\bf Logistic Bernoulli} : $Y$ Bernoulli, $p(Y{=}1\mid X) = \S(X)$ 
%\newline \phantom{{\bf Logistic Binary} :} $Y = \leftbb X-Z \geq 0 \rightbb $ 
}\\
%\hline
%\multicolumn{2}{|l|}{composition: $Y = \Heviside(X - Z)$, $Z\sim \S$}\\
\hline
mean $\mu'$ & \cref{AP2a}, \cref{AP2b}\\
%from composition, explicitly by \cref{AP2a}\\
%$\mu' = \mu \big[ \Heviside(\Noise(X)) \big]$, \newline explicitly \cref{AP2a}\\
\hline
variance & $\mu' (1-\mu')$ \\
\hline
\end{tabu}\\[-1.4pt]%%%%%%%%%%%%%%%%%%%%%%%%%%%%%%%%%%%%%%%
%\\[0.5\baselineskip]
%%
%Mapping layers\\
%
%
%\newline
%\skipmyskip
\begin{tabu}{|p{0.17\linewidth}|p{0.83\linewidth-4\tabcolsep}|}
\hline
\multicolumn{2}{|l|}{{\bf Logistic transform} : $Y = \S(X)$; \cref{detail:logistict-var} }\\
\hline
mean $\mu'$ & \cref{PEA}, \cref{AP2a}\\
\hline
variance & PEA variance~\cite[eq.14]{AstudilloN11}. %, upper bound~\cite[eq. 27]{Frey-99}
\newline
%\revisit $\Phi\big((\mu - 1) / (\sigma^2 + \sigma_\S^2-1)^{\frac{1}{2}} \big) - \mu'^2$.
$4 (1+4 \sigma^{-2})^{-1} (\mu'(1-\mu'))^2$
\\ %$\Sigma(Y) = \mu(Y)$ \\
\hline
\end{tabu}\\[-1.4pt]%%%%%%%%%%%%%%%%%%%%%%%%%%%%%%%%%%
%\\[0.5\baselineskip]
%
%\newline
%\skipmyskip
%\begin{tabu}{|p{0.17\linewidth}|p{0.83\linewidth-4\tabcolsep}|}
%\hline
%\multicolumn{2}{|p{\linewidth-2\tabcolsep}|}{{\bf Stochastic rectifier} : $Y = \max(0, X-Z)$, $Z\sim \S$ }\\
%\hline
%mean, var & From composition: $\relu( X - Z )$\\
%%\multicolumn{2}{|p{\linewidth-2\tabcolsep}|}{Composed as $Y = \relu(\AddNoise(X))$ }\\
%\hline
%\end{tabu}\\[-1.4pt]%%%%%%%%%%%%%%%%%%%%%%%%%%%%%%%%%%%%%
%\\[0.5\baselineskip]
%\newline
%\skipmyskip
%\skipmyskip
%\newline
\begin{tabu}{|p{0.15\linewidth}|>{\raggedright}p{0.85\linewidth-4\tabcolsep}|}
\hline
\multicolumn{2}{|l|}{{\bf Max} : $Y = \max(X_1, X_2)$; \cref{sec:max}}\\
\hline
%Mean, \newline Variance & ${\Abs}(\mu_1 - \mu_2, \Sigma_1 + \Sigma_2)$ \\
mean $\mu'$ & $\mu_2 + s ( x \Phi(a) + \phi(a) )$,\newline
where $s = (\sigma_1^2+\sigma_2^2)^\frac{1}{2}$ and $a = (\mu_1-\mu_2)/s$.
%$
%\begin{aligned}
%%&\E [\relu(X_1-X_2)] + \E [X_2]\\
%& \mu_1 \Phi((\mu_1{-}\mu_2)/s) + \mu_2 \Phi((\mu_2{-}\mu_1)/s) \\
%& + s \phi((\mu_1{-}\mu_2)/s),
%\end{aligned}
%$
%\newline
%where $s = (\sigma_1^2+\sigma_2^2)^\frac{1}{2}$
\\
\hline 
variance & 
$
%\sigma_1^2 \Phi(a)+ \sigma_2^2 \Phi(-a) + s^2 (\R(a) - \Phi(a))
\sigma_1^2 \Phi(a)+ \sigma_2^2 \Phi(-a) + s^2 (\R(a) - \Phi(a)) \approx
\sigma_1^2 \S(a/t) + \sigma_2^2 \S(-a/t)
%\begin{aligned}
%&(\sigma_1^2+\mu_1^2)\Phi((\mu_1-\mu_2)/s)\\
%+&(\sigma_2^2+\mu_2^2)\Phi((\mu_2-\mu_1)/s)\\
%+&(\mu_1 + \mu_2)s\phi((\mu_1-\mu_2)/s) - \mu'^2
%\end{aligned}
$,
\newline where $\R$ and $t$ are as in ReLU.
\\
\hline
%\end{tabu}\\[-1.4pt]%%%%%%%%%%%%%%%%%%%%%%%%%%%%%%%%%%%%%%%%%%%%%%%%%
%\begin{tabu}{|p{0.15\linewidth}|p{0.85\linewidth-4\tabcolsep}|}
%\hline
%\multicolumn{2}{|l|}{  {\bf Stochastic softmax}  }\\
%\hline
%%Mean, \newline Variance & ${\Abs}(\mu_1 - \mu_2, \Sigma_1 + \Sigma_2)$ \\
%\multicolumn{2}{|p{\linewidth-2\tabcolsep}|}{
%defined as: $p(Y{=}y \mid X) \propto \exp( \sum_i w_{y i} X_i ) $
%}\\
%\hline
%\multicolumn{2}{|p{\linewidth-2\tabcolsep}|}{
%$
%%\begin{aligned}
%%&
%q(y) = \Big(  1{+}\sum_{k\neq y}\exp\Big\{ \frac{\sum_{i}(w_{k i} - w_{y i})\mu_i}{ \sqrt{ \sum_{i}(w_{ki} - w_{y i})^2\sigma_i^2/\Sigma_\S+1}  } \Big\} \Big)^{-1}%\\
%%& q(y) := 1/\sum_y q(y)
%%\end{aligned}
%$,
%\newline
%renormalized, see~\cref{detail:softmax}.
%}\\
%%\hline 
%%variance & 
%%not defined
%%\\
%\hline
\end{tabu}\\[-1.4pt]%%%%%%%%%%%%%%%%%%%%%%%%%%%%%%%%%%%%%%%%%%%%%%%%
\begin{tabu}{|p{0.15\linewidth}|p{0.85\linewidth-4\tabcolsep}|}
\hline
\multicolumn{2}{|l|}{  {\bf Softmax}
: $p(Y{=}y | X) = e^{X_y}/\sum_{k} e^{X_k}$;~\cref{detail:softmax}  }\\
\hline
%Mean, \newline Variance & ${\Abs}(\mu_1 - \mu_2, \Sigma_1 + \Sigma_2)$ \\
%\multicolumn{2}{|p{\linewidth-2\tabcolsep}|}{}\\
%\hline
\multicolumn{2}{|p{\linewidth-2\tabcolsep}|}{
Normal: 
$
q(y) = \prod_{k \neq y} \Phi \Big(\frac{\mu_y-\mu_k}{\sqrt{\sigma^2_y + \sigma^2_k + \sigma_S^2}}\Big)
$,
 renormalized.\newline
Logistic: 
$
q(y) = \Big(\sum_{k}\exp\Big\{ \frac{\mu_k - \mu_y}{ \sqrt{ (\sigma_k^2+\sigma_y^2)/\sigma_\S^2+1} } \Big\} \Big)^{-1}%\\s
$, renormalized.
\newline
Simplified: $q(y) = {\rm softmax_k}(\mu_k/\sqrt{\sigma_k^2/\sigma_S^2+1})$.
}\\
\hline
\end{tabu}\\[-1.4pt]%%%%%%%%%%%%%%%%%%%%%%%%%%%%%%%%%%%%%%%%%%%%%%%%
%\end{tabular}
%\caption{Approximations for propagating moments through common layers. Functions $\phi$ and $\Phi$ denote respectively the pdf and the cdf of a standard normal distribution. The uni-variate functions $\Phi$, $\R$, $\Li_2$ can be well approximated. Some additional layers are given in~\cref{sec:additional-layers}. Logistic-Bernoulli and softmax layers are inherently stochastic, other layers can be made stochastic by injecting noise.
%\label{table:approx}}
%\end{table}
%
%\setlength{\tabcolsep}{4pt}
%\tabulinesep=3pt
%\par\noindent
%
\begin{tabu}{|p{0.17\linewidth}|p{0.83\linewidth-4\tabcolsep}|}
\hline
\multicolumn{2}{|p{\linewidth-2\tabcolsep}|}{
{\bf Product} : $Y  = X_1 X_2$
}\\
\hline
\multicolumn{2}{|l|}{
mean: $\mu_1 \mu_2$, variance: $\sigma_1^2 \sigma_2^2 + \sigma_1^2 \mu_2^2 + \mu_1^2 \sigma_2^2$
}\\
\hline
\end{tabu}\\[-1.4pt]%%%%%%%%%%%%%%%%%%%%%%%%%%%%%%%%%%%%%%%%%%%
\begin{tabu}{|p{0.17\linewidth}|p{0.83\linewidth-4\tabcolsep}|}
\hline
\multicolumn{2}{|l|}{{\bf Abs} : $Y = |X|$}\\
\hline
mean & 
$
\begin{aligned}
\mu' = 2 \E [ \relu(X) ]-\mu = 2 \mu \Phi(\mu/\sigma)+2\sigma \phi(-\mu/\sigma) - \mu
\end{aligned}
$
\\
\hline
variance & $\mu^2 + \sigma^2 - \mu'^2$ \\
\hline
\end{tabu}\\[-1.4pt]%%%%%%%%%%%%%%%%%%%%%%%%%%%%%%%%%%%%%%%%%%%
\begin{tabu}{|p{0.17\linewidth}|p{0.83\linewidth-4\tabcolsep}|}
\hline
\multicolumn{2}{|p{\linewidth-2\tabcolsep}|}{{\bf Probit} : $Y$ Bernoulli with $p(Y{=}1\mid X) = \Phi(X)$
}\\
\hline
\multicolumn{2}{|l|}{composition: $Y = \Heviside(X - Z)$, $Z\sim \Phi$}\\
\hline
mean $\mu'$ & $\Phi(\mu/\sqrt{\sigma^2+1})$, \cite{Frey-99}, \cref{sec:probit}.\\
\hline
variance & $\mu' (1-\mu')$ \\
\hline
\end{tabu}\\[-1.4pt]%%%%%%%%%%%%%%%%%%%%%%%%%%%%%%%%%%%%%%%
\begin{tabu}{|p{0.17\linewidth}|p{0.83\linewidth-4\tabcolsep}|}
\hline
\multicolumn{2}{|l|}{{\bf Normal cdf transform} : $Y = \Phi(X)$, \cite{Frey-99}}\\
\hline
mean & $\Phi(\mu/\sqrt{\sigma^2+1})$
\\
\hline
variance & upper bound \cite[eq. 27]{Frey-99}\\
\hline
\end{tabu}\\[-1.4pt]%%%%%%%%%%%%%%%%%%%%%%%%%%%%%%%%%%%%%%%%%%%

%% file: main.bbl
\begin{thebibliography}{36}
\providecommand{\natexlab}[1]{#1}
\providecommand{\url}[1]{\texttt{#1}}
\providecommand{\urlprefix}{}

\bibitem[{Arpit et~al.(2016)Arpit, Zhou, Kota, and Govindaraju}]{ArpitZKG16}
Arpit, D., Zhou, Y., Kota, B.U., Govindaraju, V.: Normalization propagation:
  {A} parametric technique for removing internal covariate shift in deep
  networks.
\newblock In: Balcan, M., Weinberger, K.Q. (eds.) ICML. {JMLR} Workshop and
  Conference Proceedings, vol.~48, pp. 1168--1176. JMLR.org (2016),
  \urlprefix\url{http://jmlr.org/proceedings/papers/v48/arpitb16.html}

\bibitem[{Astudillo and da~Silva~Neto(2011)}]{AstudilloN11}
Astudillo, R.F., da~Silva~Neto, J.P.: Propagation of uncertainty through
  multilayer perceptrons for robust automatic speech recognition.
\newblock In: INTERSPEECH (2011)

\bibitem[{Bengio et~al.(2013)Bengio, L{\'e}onard, and
  Courville}]{Bengio2013EstimatingOP}
Bengio, Y., L{\'e}onard, N., Courville, A.C.: Estimating or propagating
  gradients through stochastic neurons for conditional computation.
\newblock CoRR abs/1308.3432 (2013)

\bibitem[{Clevert et~al.(2015)Clevert, Unterthiner, and
  Hochreiter}]{ClevertUH15}
Clevert, D.A., Unterthiner, T., Hochreiter, S.: Fast and accurate deep network
  learning by exponential linear units ({ELUs}).
\newblock CoRR abs/1511.07289 (2015)

\bibitem[{Dayan et~al.(1995)Dayan, Hinton, Neal, and Zemel}]{Dayan-95}
Dayan, P., Hinton, G.E., Neal, R.N., Zemel, R.S.: The {Helmholtz} machine.
\newblock Neural Computation 7, 889--904 (1995)

\bibitem[{Fawzi et~al.(2016)Fawzi, Moosavi-Dezfooli, and
  Frossard}]{Fawzi-16-robustness}
Fawzi, A., Moosavi-Dezfooli, S.M., Frossard, P.: Robustness of classifiers:
  from adversarial to random noise.
\newblock In: NIPS, pp. 1632--1640 (2016)

\bibitem[{Flach et~al.(2017)Flach, Shekhovtsov, and Fikar}]{Flach-17}
Flach, B., Shekhovtsov, A., Fikar, O.: Generative learning for deep networks.
\newblock CoRR abs/1709.08524 (2017)

\bibitem[{Frey and Hinton(1999)}]{Frey-99}
Frey, B.J., Hinton, G.E.: Variational learning in nonlinear gaussian belief
  networks.
\newblock Neural Comput. 11(1), 193--213 (Jan 1999)

\bibitem[{Goodfellow et~al.(2015)Goodfellow, Shlens, and
  Szegedy}]{Goodfellow-15-adversarial}
Goodfellow, I., Shlens, J., Szegedy, C.: Explaining and harnessing adversarial
  examples.
\newblock In: International Conference on Learning Representations (2015),
  \urlprefix\url{http://arxiv.org/abs/1412.6572}

\bibitem[{Hern\'{a}ndez-Lobato and Adams(2015)}]{Hernandez-15-PBP}
Hern\'{a}ndez-Lobato, J.M., Adams, R.P.: Probabilistic backpropagation for
  scalable learning of {Bayesian} neural networks.
\newblock In: ICML. pp. 1861--1869 (2015)

\bibitem[{Ioffe and Szegedy(2015)}]{IoffeS15}
Ioffe, S., Szegedy, C.: Batch normalization: Accelerating deep network training
  by reducing internal covariate shift.
\newblock In: ICML. vol.~37, pp. 448--456 (2015)

\bibitem[{Kingma(2013)}]{Kingma13-fast}
Kingma, D.P.: Fast gradient-based inference with continuous latent variable
  models in auxiliary form.
\newblock CoRR abs/1306.0733 (2013)

\bibitem[{Kingma and Ba(2014)}]{KingmaB14}
Kingma, D.P., Ba, J.: Adam: {A} method for stochastic optimization.
\newblock CoRR abs/1412.6980 (2014),
  \urlprefix\url{http://arxiv.org/abs/1412.6980}

\bibitem[{Kingma et~al.(2015)Kingma, Salimans, and Welling}]{Kingma-15-dropout}
Kingma, D.P., Salimans, T., Welling, M.: Variational dropout and the local
  reparameterization trick.
\newblock In: Advances in Neural Information Processing Systems 28, pp.
  2575--2583 (2015)

\bibitem[{Kingma and Welling(2014{\natexlab{a}})}]{Kingma-14-Gradient-based}
Kingma, D.P., Welling, M.: Efficient gradient-based inference through
  transformations between {B}ayes nets and neural nets.
\newblock In: ICML. pp. II--1782--II--1790 (2014{\natexlab{a}}),
  \urlprefix\url{http://dl.acm.org/citation.cfm?id=3044805.3045091}

\bibitem[{Kingma and Welling(2014{\natexlab{b}})}]{Kingma:2014}
Kingma, D.P., Welling, M.: Efficient gradient-based inference through
  transformations between {Bayes} nets and neural nets.
\newblock In: ICML. pp. II--1782--II--1790. ICML'14, JMLR.org
  (2014{\natexlab{b}}),
  \urlprefix\url{http://dl.acm.org/citation.cfm?id=3044805.3045091}

\bibitem[{Klambauer et~al.(2017)Klambauer, Unterthiner, Mayr, and
  Hochreiter}]{Klambauer-SELU}
Klambauer, G., Unterthiner, T., Mayr, A., Hochreiter, S.: Self-normalizing
  neural networks.
\newblock CoRR abs/1706.02515 (2017)

\bibitem[{MacKay(1992{\natexlab{a}})}]{MacKay-92-classification}
MacKay, D.J.C.: The evidence framework applied to classification networks.
\newblock Neural Computation 4(5), 720--736 (Sept 1992{\natexlab{a}})

\bibitem[{MacKay(1992{\natexlab{b}})}]{MacKay-92-Bayesian}
MacKay, D.J.C.: A practical {Bayesian} framework for backpropagation networks.
\newblock Neural Comput. 4(3), 448--472 (May 1992{\natexlab{b}}),
  \urlprefix\url{http://dx.doi.org/10.1162/neco.1992.4.3.448}

\bibitem[{Malik and Abraham(1973)}]{Malik-73}
Malik, H.J., Abraham, B.: Multivariate logistic distributions.
\newblock The Annals of Statistics 1(3), 588--590 (1973),
  \urlprefix\url{http://www.jstor.org/stable/2958123}

\bibitem[{Minka(2001)}]{Minka-2001}
Minka, T.P.: Expectation propagation for approximate {Bayesian} inference.
\newblock In: Uncertainty in Artificial Intelligence. pp. 362--369 (2001)

\bibitem[{Moosavi-Dezfooli et~al.(2017)Moosavi-Dezfooli, Fawzi, Fawzi, and
  Frossard}]{Moosavi-Dezfooli_2017_CVPR}
Moosavi-Dezfooli, S.M., Fawzi, A., Fawzi, O., Frossard, P.: Universal
  adversarial perturbations.
\newblock In: CVPR (July 2017)

\bibitem[{Nadarajah and Kotz(2008)}]{NadarajahK08}
Nadarajah, S., Kotz, S.: Exact distribution of the max/min of two gaussian
  random variables.
\newblock IEEE Trans. VLSI Syst. 16(2), 210--212 (2008)

\bibitem[{Neal(1992)}]{Neal:1992}
Neal, R.M.: Connectionist learning of belief networks.
\newblock Artif. Intell. 56(1), 71--113 (Jul 1992)

\bibitem[{Nguyen et~al.(2015)Nguyen, Yosinski, and Clune}]{nguyen2015deep}
Nguyen, A., Yosinski, J., Clune, J.: Deep neural networks are easily fooled:
  High confidence predictions for unrecognizable images.
\newblock In: CVPR (2015)

\bibitem[{Papaspiliopoulos et~al.(2007)Papaspiliopoulos, Roberts, and
  Sköld}]{papaspiliopoulos2007}
Papaspiliopoulos, O., Roberts, G.O., Sköld, M.: A general framework for the
  parametrization of hierarchical models.
\newblock Statist. Sci. 22(1), 59--73 (02 2007),
  \urlprefix\url{http://dx.doi.org/10.1214/088342307000000014}

\bibitem[{Pearl(1988)}]{Pearl-88}
Pearl, J.: Probabilistic Reasoning in Intelligent Systems: Networks of
  Plausible Inference (1988)

\bibitem[{Rezende et~al.(2014)Rezende, Mohamed, and Wierstra}]{rezende14}
Rezende, D.J., Mohamed, S., Wierstra, D.: Stochastic backpropagation and
  approximate inference in deep generative models.
\newblock In: ICML. vol.~32, pp. 1278--1286 (2014)

\bibitem[{Rodner et~al.(2016)Rodner, Simon, Fisher, and Denzler}]{Rodner16_FRN}
Rodner, E., Simon, M., Fisher, B., Denzler, J.: Fine-grained recognition in the
  noisy wild: Sensitivity analysis of convolutional neural networks approaches.
\newblock In: BMVC (2016)

\bibitem[{Salimans and Kingma(2016)}]{Salimans2016WeightNA}
Salimans, T., Kingma, D.P.: Weight normalization: A simple reparameterization
  to accelerate training of deep neural networks.
\newblock In: NIPS (2016)

\bibitem[{Schoenholz et~al.(2016)Schoenholz, Gilmer, Ganguli, and
  Sohl-Dickstein}]{Schoenholz2016DeepIP}
Schoenholz, S.S., Gilmer, J., Ganguli, S., Sohl-Dickstein, J.: Deep information
  propagation.
\newblock CoRR abs/1611.01232 (2016)

\bibitem[{Shekhovtsov and Flach(2018)}]{shekhovtsov-18-norm}
Shekhovtsov, A., Flach, B.: Normalization of neural networks using analytic
  variance propagation.
\newblock In: CVWW (2018)

\bibitem[{Srivastava et~al.(2014)Srivastava, Hinton, Krizhevsky, Sutskever, and
  Salakhutdinov}]{srivastava14a}
Srivastava, N., Hinton, G., Krizhevsky, A., Sutskever, I., Salakhutdinov, R.:
  Dropout: A simple way to prevent neural networks from overfitting.
\newblock Journal of Machine Learning Research 15, 1929--1958 (2014),
  \urlprefix\url{http://jmlr.org/papers/v15/srivastava14a.html}

\bibitem[{Szegedy et~al.(2014)Szegedy, Zaremba, Sutskever, Bruna, Erhan,
  Goodfellow, and Fergus}]{Szegedy-14-intriguing}
Szegedy, C., Zaremba, W., Sutskever, I., Bruna, J., Erhan, D., Goodfellow, I.,
  Fergus, R.: Intriguing properties of neural networks.
\newblock In: International Conference on Learning Representations (2014),
  \urlprefix\url{http://arxiv.org/abs/1312.6199}

\bibitem[{Wang and Manning(2013)}]{wang2013fast}
Wang, S., Manning, C.: Fast dropout training.
\newblock In: ICML. pp. 118--126 (2013)

\bibitem[{Williams(1992)}]{Williams1992}
Williams, R.J.: Simple statistical gradient-following algorithms for
  connectionist reinforcement learning.
\newblock Machine Learning 8(3), 229--256 (May 1992),
  \urlprefix\url{https://doi.org/10.1007/BF00992696}

\end{thebibliography}
